%% file: mainJarxiv.tex
\documentclass{article}

\usepackage[utf8]{inputenc} % allow utf-8 input
\usepackage{hyperref}       % hyperlinks
\usepackage{url}            % simple URL typesetting
\usepackage{booktabs}       % professional-quality tables
\usepackage{amsfonts}       % blackboard math symbols
\usepackage{nicefrac}       % compact symbols for 1/2, etc.
\usepackage{microtype}      % microtypography

\input{macro}
\title{Fast Low-Rank Matrix Estimation without the Condition Number}

\author{Mohammadreza Soltani and Chinmay Hegde\\
  Department of Electrical and Computer Engineering\\
  Iowa State University\thanks{\textbf{Funding: }This work was supported in part by grants from the National Science Foundation (CCF-1566281) and NVIDIA.}
}

\begin{document}
\date{ }
\maketitle

\begin{abstract}
In this paper, we study the general problem of optimizing a convex function $F(L)$ over the set of $p \times p$ matrices, subject to rank constraints on $L$. However, existing first-order methods for solving such problems either are too slow to converge, or require multiple invocations of singular value decompositions. On the other hand, factorization-based non-convex algorithms, while being much faster, require stringent assumptions on the \emph{condition number} of the optimum. In this paper, we provide a novel algorithmic framework that achieves the best of both worlds: asymptotically as fast as factorization methods, while requiring no dependency on the condition number. 

We instantiate our general framework for three important matrix estimation problems that impact several practical applications; (i) a \emph{nonlinear} variant of affine rank minimization, (ii) logistic PCA, and (iii) precision matrix estimation in probabilistic graphical model learning. We then derive explicit bounds on the sample complexity as well as the running time of our approach, and show that it achieves the best possible bounds for both cases. We also provide an extensive range of experimental results, and demonstrate that our algorithm provides a very attractive tradeoff between estimation accuracy and running time.
\end{abstract}

\input{intro}

\input{prior}

\input{model}

\input{res}

\newpage
\small
\bibliographystyle{unsrt}
\bibliography{mrsbiblio.bib,chinbiblio.bib,csbib.bib,kernels.bib}

\newpage
\input{appen}

\end{document}

%% file: macro.tex
\usepackage{amsmath}
\usepackage{amssymb}
\usepackage{amsthm}
\usepackage{algorithm}
\usepackage{algorithmic}
\usepackage{fullpage}
\usepackage{tikz}
\usepackage{tabu}
\usepackage{graphics}
\usepackage{subfig}
\usepackage{epstopdf}
\usepackage{epsfig}
\usepackage[bold,full]{complexity}
\usepackage{caption}
\usepackage{float}
\usepackage{pgfplotstable}
\usepackage{pgfplots}
\newtheorem{theorem}{Theorem}%[section]           %[section]   %numberes automatically
\newtheorem{lemma}[theorem]{Lemma}         %[theorem] 
\newtheorem{corollary}[theorem]{Corollary}     %{corollary}{Corollary}[theorem]   

\newtheorem{definition}[theorem]{Definition}

\def\x{{\mathbf x}}
\def\R{{\mathbb{R}}}

\def\rank{\text{rank}}
\def\O{{\mathcal{O}}}
\def\P{{\mathcal{P}}}

\def\la{\lambda}
\def\si{\sigma}
\def\A{\mathcal{A}}

\def\Ta{{\mathcal{T}}}

\def\M{\mathcal{M}}
\def\U{\mathbb{U}}

\def\bS{\bar{S}}
\def\'{\prime}
\def\Otilde{{\widetilde{O}}}

\usepackage{enumitem}

\usepackage{todonotes}
\def\poly{\text{poly}}

%% file: intro.tex
\section{Introduction}
\label{sec:intro}

In this paper, we consider the following optimization problem:
%\begin{equation}
\begin{align}
\underset{L}{\text{min}}
&~~F(L)  \label{opt_prob} \\
 \text{s.t.} 
&~~\rank(L)\leq r^*, \nonumber
\end{align}
%\end{equation}
where $F(L):\R^{p\times p}\rightarrow\R$ is a convex smooth function defined over matrices $L\in\R^{p\times p}$ with rank $r^*\ll p$.\footnote{For convenience, all our matrix variables will be of size $p \times p$, but our results extend seamlessly to rectangular matrices.}
This problem has recently received significant attention in machine learning, statistics, and signal processing~\cite{chen2015fast,udell2016generalized}. Several applications of this problem abound in the literature, including affine rank minimization~\cite{recht2010guaranteed,tu2016low,jain2010guaranteed}, matrix completion~\cite{candes2009exact}, and collaborative filtering~\cite{jain2013low}. 
%Problem~\eqref{opt_prob} also appears in the context of learning shallow polynomial neural networks~\cite{livni2014computational}, and rigorous solutions to~\eqref{opt_prob} sheds light on developing a non-asymptotic algorithmic understanding of training such networks. %,zhong2017recovery

In most of the above applications, $F(L)$ is typically assumed to be a smooth, quadratic function (such as the squared error). For instance, in machine learning, the squared loss between the pair of observed and predicted outputs would be a natural choice, and indeed most of the works in the matrix estimation literature focus on optimizing such functions. 

But there are many cases in which other loss functions are used. For example, in neural network learning, the loss function is usually chosen according to the negative cross-entropy between the distributions of the fitted model and the training samples~\cite{Goodfellow-et-al-2016}. 
As another example, in graphical model learning, the goal is usually to estimate the covariance/precision matrix. In this case, the negative log-likelihood function is an appropriate choice. As a third example, in the problem of one-bit matrix completion~\cite{davenport20141} or logistic PCA~\cite{park2016finding}, $F(L)$ is modeled, again, as the log-likelihood of the observations. %In all these cases, we encounter with the mathematical optimization problem in the form of~\eqref{opt_prob}.   

From the computational perspective, the traditional approach is to adopt first-order optimization for solving \eqref{opt_prob}. Several different approaches (with theoretical guarantees) have been proposed in recent years. The first group of these methods are related to the convex methods in which the rank constraint is relaxed by the nuclear norm proxy~\cite{fazel2002matrix}, resulting the overall convex problem which can be solved by off-the-shelf solvers. While these methods achieve the best sample complexity, i.e., the minimum required number of samples for achieving the small estimation error, they are computationally expensive and the overall running time can be slow if $p$ is very large. 

To alleviate this issue, several non-convex methods have been proposed based on using non-convex regularizers. Non-convex iterative methods~\cite{jain2010guaranteed,jain2014iterative,quanming2017large} typically require less computational complexity per iteration. On the other hand, factorized gradient methods~\cite{chen2015fast,bhojanapalli2016dropping,tu2016low,wang2017unified} are computationally very appealing since they reduce the number of variables from $p^2$ to $pr$ by writing $L$ as $L=UV^T$ where $U,V\in\R^{p\times r}$ and $r\ll p$, and removing the rank constraint from problem~\eqref{opt_prob}. 
%The singular value projection algorithms~\cite{j} use SVD as the projection step within the gradient descent framework in each iteration, and they are more robust to the spectral properties of the underlying matrix. 

However, the overwhelming majority of existing methods suffer from one or several of the following problems: their convergence rate is slow (typically, sublinear); the computational cost per iteration is high, involving invocations of singular value decomposition; or they have stringent assumptions on the {spectral properties} (such as the condition number) of the solution to~\eqref{opt_prob}. 

Our goal in this paper is to propose an algorithm to alleviate the above problems simultaneously. Specifically, we seek an algorithm that exhibits:
%\begin{enumerate}
%\item 
(i) \emph{linearly} fast convergence, 
%\item 
(ii) \emph{computationally efficient} per iteration, 
%\item 
(iii) works for a \emph{broad} class of loss functions, and
%\item 
(iv) \emph{robust} to effects such as matrix condition number. 
%\end{enumerate}

\subsection{Our contributions}

In this paper, we propose a general \emph{non-convex} algorithmic framework, that we call \emph{MAPLE}, for solving problems of the form~\eqref{opt_prob} for objective functions that satisfy the commonly-studied \emph{Restricted Strongly Convex/Smooth} ({RSC/RSS}) conditions. Our algorithmic approach enjoys the following benefits:  

{\textit{\textbf{Linear convergence.}} We provide rigorous analysis to show that our proposed algorithm enjoy a linear convergence rate (no matter how it is initialized). 

{\textit{\textbf{Fast per-iteration running time}}}. We provide rigorous analysis to show that our algorithm exhibits fast per-iteration running time. Our method (per-iteration) leverages recent advances in randomized low-rank approximation methods, and their running time is close to optimal for constant $r$\footnote{Our approach is akin to the previous work of~\cite{becker2013randomized}, but strictly improves upon this approach in terms of sample complexity.}.

{\textit{\textbf{No limitations on strong convexity/smoothness constants.}} In a departure from the majority of the matrix optimization literature, our algorithm succeeds under no particular assumptions on the \emph{extent} to which the objective function $F$ is strongly smooth/convex. % provided that the subspace we are projecting onto has large enough dimension 
(These are captured by properties known as \emph{restricted strong convexity} and \emph{smoothness}, which we elaborate below.) 

{\textit{\textbf{No dependence on matrix condition numbers.}} In contrast with several other results in the literature, our proposed algorithm does not depend on stringent assumptions on the condition number (i.e., the ratio of maximum to minimum nonzero singular values) of the solution to~\eqref{opt_prob}.

%{\textit{\textbf{Improvement.}} Our algorithm is based on using the idea of approximate projection, and we show that its linear convergence is guaranteed in the optimal sample complexity regime as opposed to the previous work by~\cite{becker2013randomized}. 
%\textbf{Instantiating in three applications.}  
{\textit{\textbf{Instantiation in applications.}} We instantiate our MAPLE framework to three important and practical applications; nonlinear affine rank minimization, logistic PCA, and precision matrix estimation in probabilistic graphical model learning. 
%
%In each of all of these applications we derive the sample complexity, calculate the running time, and analyze the statistical error rate of the underlying models. 
%In addition, for NLARM, we show that by choosing an appropriate design operator $\A$ (defined later), we can reduce the computational complexity of calculating the gradient in each step significantly which makes the overall algorithm a promising approach for very large size nonlinear matrix sensing problem.

%Putting together these ingredients, we get the first \emph{condition-free}, almost-linear time algorithm for solving~\eqref{opt_prob} for a wide range of instances.

\subsection{Techniques}
\label{ourtechnique}
Our approach is an adaptation of the algorithm proposed in~\cite{jain2014iterative}. That is an iterative approach that alternates between taking a gradient descent step and thresholding the largest singular values of the optimization variable. The key idea of that work is that each gradient update is projected onto the space of matrices with rank $r$ that is \emph{larger} than $r^*$, the rank parameter in Problem~\eqref{opt_prob}. This trick can greatly alleviate situations where the objective function exhibits poor restricted strong convexity/smoothness properties; more generally,  
%(or RIP constant in the linear case model), 
the overall algorithm can be applied to ill-posed problems. However, their algorithm requires performing a full exact singular value decomposition (SVD) after each gradient descent step. This results in poor overall running time, as the per-iteration cost is \emph{cubic} ($\O(p^3)$) in the matrix dimension\footnote{Here, one may argue that the running time of the approach in~\cite{jain2014iterative} and the other IHT-type algorithms take $\O(p^2r)$ time using truncated SVD (via power iteration or similar). Unfortunately, this is not technically true and seems to be a common misconception in several low-rank matrix recovery papers. Finding a truncated SVD of a matrix only takes $\O(p^2r)$ time if the input matrix is exactly rank $r$; more generally, the running time of power method-like iterative approaches scales as $\O(p^2r/gap)$ where the denominator denotes the ratio of the $r^{th}$ and $(r+1)^{th}$ singular values which can be very small, and consequently inflates the running time to $\O(p^3)$ time.}.

Our method resolves this issue by replacing the exact SVD with a gap-independent \emph{approximate} low-rank projection, while still retaining the idea of projecting onto a larger space. To establish soundness of our approach, we establish a property about (approximate) singular value projections, extending recent new results from non-convex optimization ~\cite{shen2016tight,li2016Nonconvex}. In particular, we prove a new structural result for an $\epsilon$-approximate projection onto the space of rank-$r$ matrices. We prove that such an approximate projection is \emph{nearly non-expansive}, and therefore enjoys similar convergence guarantees as convex projected gradient descent.

To be more precise, we know that for any matrix $A$ and rank-$r'$ matrix $B$, the best rank-$r'$ approximation of $A$, denoted by $H_{r'}(A)$ satisfies the following:
$$ \|H_{r'}(A) - B\|_F\leq2\|A - B\|_F,$$
This bound is very loose (following a simple application of the triangle inequality) and the upper bound is, in fact, never achieved~\cite{shen2016tight}. We prove that the approximation factor 2 can be sharpened to close to $1$ if we use a rank parameter that is sufficiently larger than $r'$. In particular, if $\Ta$ is an $\epsilon$-approximate singular value projection operator, we prove that: 
\begin{align*}
\|\Ta(A) - B\|_F^2\leq\left(1+\frac{2}{\sqrt{1-\epsilon}}\frac{\sqrt{r'}}{\sqrt{r-r'}}\right)\|A - B\|_F^2,
\end{align*}
%\red{the frobenius norms are not squared. inconsistent with later theorem.}
where $r>r'$, $\rank(B) = r'$ and $\Ta$ implements an $\epsilon$-approximate projection onto the set of matrices with rank-$r$. Therefore, by increasing $r$, we (nearly) recover the non-expansivity property of projection, and this helps prove strong results about our proposed projected gradient descent scheme. 

%While replacing of exact SVD with approximate one has been already proposed by~\cite{becker2013randomized} for the linear observation model and under the RIP condition, their theoretical guarantees is very restricted. That is, in the regime of optimal sample complexity, i.e., $n=\O(pr)$, their approximate projection should be applied onto a a matrix with rank as the order of $p$ in order to have convergence. This restriction overshadows the usage of any approximate SVD method instead of the exact one.
Integrating the above result into (projected) gradient descent gives linear convergence of the proposed algorithm for a very broad class of objective functions. Since we use approximate low-rank projections, the running time of the projection step is (almost) linear in the size of the matrix if $r^*$ is sub-linear in $n$.

\subsection{Stylized applications}
We also instantiate our MAPLE framework to three applications of practical interest. 

First, we consider a problem that we call \emph{nonlinear affine rank minimization} (NLARM). Formally, we consider an observation model akin to the Generalized Linear Model (GLM)~\cite{kakade2011}:
$$
y = g(\A (L^*)) + e,
$$
where $g$ denotes a nonlinear \emph{link} function, $\A$ denotes a linear observation operator, which we formally define later, and $e\in\R^m$ denotes an additive noise vector. The goal is to reconstruct $L^*$ from $y$, given that $L^*$ is of rank at most $r^*$. For this application, we derive the sample complexity of our algorithm, calculate the running time, and analyze statistical error rates. More specifically, we define an specific objective function tailored to $g$ and verify that it is strongly convex/smooth; moreover, we show that $\widetilde{O}(pr^*)$ samples is enough to estimate $L^*$ up to the noise level, and this matches those of the best available methods. %Our technique for deriving sample complexity is based the $\epsilon$-net for the set of low-rank matrices, and designing sensing operator $\A$ based on Johnson-Lindenstrauss lemma. 
In addition, the running time required to estimate $L^*$ scales as $\Otilde(p^2 r^*)$, which is nearly linear with the size of $L^*$ and independent of all other spectral properties of $L^*$ (such as its condition number). This marks a strict improvement over all other comparable existing methods.

Second, we discuss the problem of \emph{logistic PCA}~\cite{park2016finding} in which we observe a binary matrix $Y$ with entries belonging to ${\{}0, 1{\}}$ such that the likelihood of each $Y_{ij}$ is given by $P(Y_{ij} = 1|L_{ij}) = \sigma(L^*_{ij})$ where $\sigma(x)  = \frac{1}{1+\exp(-x)}$ is a sigmoidal nonlinearity. The goal is to estimate an underlying low-rank matrix $L^*$ by trying to find the solution of following optimization problem:
 $$F(L) = -\sum_{i,j}\big{(}Y_{ij}\log\sigma(L_{ij}) + (1-Y_{ij})\log(1-\sigma(L_{ij}))\big{)}.$$ Again, we show how to use our framework to solve this problem with nearly linear running time.

Third, we instantiate our framework in the context of \emph{precision matrix estimation} in probabilistic graphical models. Specifically, the goal is to estimate a low-rank precision matrix $L^*$ based on observed samples $X_i\in\R^p$ for $i=1,\ldots,n$. In this setup, the objective function $F(L)$ is given by the negative log likelihood of the observed samples. 
%ur technique for verifying RSC/RSS conditions of NLL is according to a key observation which states $F(L)$ is globally strongly convex, and when restricted to any compact psd cone, it also satisfies strong smoothness condition. As a result of our analysis, we bound RSS/RSC constants of $F(L)$ which is a non-trivial task and considerably different from the the existing methods and consequently additional effort is required due to the several properties of the true precision matrix $L^*$.  
We show that with $n = \O(pr)$ independent samples, the proposed algorithm returns an estimate up to constant error, and once again, our algorithm exhibits nearly linear running time, independent of how poorly the underlying precision matrix is conditioned. Moreover, we show that the our algorithm provide the best empirical performance (in terms of estimation error) among available competing methods.

%% file: prior.tex
\section{Prior Work}
\label{sec:prior}

Optimization problems with rank constraints arise in several different applications; examples include robust PCA~\cite{candes2011rpca,Venkat2009sparse,netrapalli2014non,yi2016fast}, precision matrix estimation using graphical models~\cite{hsieh2014quic,Venkat2010latent}, phase retrieval~\cite{candes2015phase,candes2013phaselift,netrapalli2013phase}, finding the square root of a PSD matrix~\cite{jain2015computing}, dimensionality reduction techniques~\cite{johnson2014logistic,schein2003generalized}, video denoising~\cite{ji2010robust}, subspace clustering~\cite{liu2013robust}, face recognition~\cite{yang2017nuclear} and many others. We only provide a subset of relevant references here; please refer to the recent survey~\cite{davenport2016overview} and references therein for a more comprehensive discussion.
%Beyond specific applications, solving~\eqref{opt_prob} as efficiently as possible has attracted considerable interest in the optimization community. 

In general, most optimization approaches to solve \eqref{opt_prob} can be categorized in four groups. In the first group of approaches, the non-convex rank constraint is relaxed into a \emph{nuclear norm} penalty, which results in a convex problem and can be solved by off-the-shelf solvers such as SDP solvers~\cite{cvx}, singular value thresholding and its accelerated versions~\cite{recht2010guaranteed,cai2010singular,GoldsteinStuderBaraniuk:2014}, and active subspace selection methods~\cite{hsieh2014nuclear}. While convex methods are well-known, their usage in the high dimensional regime is prohibitive (incurring cubic, or worse, running time).

The second group of approaches includes non-convex methods, replacing the rank constraint with a more tractable \emph{non-convex} regularizer instead of the nuclear norm. 
These include regularization with the smoothly clipped absolute deviation (SCAD)~\cite{fan2001variable}, and iteratively re-weighted nuclear norm (IRNN) minimization~\cite{lu2016nonconvex}. 
While these approaches can reduce the computational cost per iteration, from $p^3$ to $p^2r$, 
they exhibit sub-linear convergence, and are quite slow in high dimensional regimes; see~\cite{quanming2017large} for details.

The third group of approaches try to solve the non-convex optimization problem~\eqref{opt_prob} based on the factorization approach of~\cite{burer2003nonlinear}. In these algorithms, the rank-$r$ matrix $L$ is factorized as $L = UV^T$, where $U,V\in\R^{p\times r}$. Using this idea removes the difficulties caused by the non-convex rank constraint; however, the objective function is not convex anymore. Nevertheless, under certain conditions, such methods succeed and have recently gained in popularity in the machine learning literature, and several papers have developed provable linear-convergence guarantees for both squared and non-squared loss functions
~\cite{tu2016low,bhojanapalli2016dropping,park2016non,chen2015fast,zheng2015convergent,jain2013low}. 

Such methods are currently among the fastest available in terms of running time. However, a major drawback is that they may require a careful spectral initialization that usually involves one or multiple full singular value decompositions (SVDs). 
To our knowledge, only three recent works in the matrix recovery literature require no full SVDs for their initialization: \cite{bhojanapalli2016global,ge2016matrix,ge2017no}. However,~\cite{bhojanapalli2016global} only discusses about the linear matrix sensing problem, and it only applies to the squared loss which requires that sensing matrix satisfies RIP condition, while our stylized application is for general non-squared loss functions with no assumption for the upper bound of $\frac{M}{m}$ (the ratio of RSS to RSC constant). Also, \cite{ge2016matrix,ge2017no} makes stringent assumptions on the coherence and other spectral properties of ground-truth matrix. For example, the running time of saddle-avoiding local search algorithm used in~\cite{ge2017no} shows polynomial dependency on the condition number (i.e., the ratio of the largest to the smallest non-zero singular values). Furthermore, the instantiation to the linear matrix sensing problem shows strict upper bound on the RIP constant. As a result, the convergence rate depends heavily on the condition number as well as other spectral properties of the optimum. Hence, if the problem is somehow poorly conditioned, their sample complexity and running time can blow up by a significant amount. 

The fourth class of methods also includes non-convex methods. Unlike the factorized methods, they do not factorize the optimization variable, $L$, but instead use low-rank projections within classical gradient descent. This approach, also called singular value projection (SVP) or iterative hard thresholding, was introduced by~\cite{jain2010guaranteed} for matrix recovery from linear measurements, and was later modified for general M-estimation problems with well-behaved objective functions~\cite{jain2014iterative}. These methods require multiple invocations of exact singular value decompositions (SVDs). While their computational complexity can be cubic in $p$ (see footnote in section~\ref{ourtechnique}), and consequently very slow in very large-scale problems, these methods do not depend on the condition number of the optimum, and in this sense are more robust than factorized methods.  
A similar algorithm to SVP as proposed by~\cite{becker2013randomized} for the squared loss case, which replaces the exact SVD with an approximate one. However, their theoretical guarantees is very restrictive which overshadows any advantage of using an approximate SVD algorithm instead of an exact SVD. That is, in the regime of optimal sample complexity, i.e., $n=\O(pr)$, their approximate projection should be applied onto a  matrix with rank as the order of $p$ in order to have convergence. Furthermore, while the idea of projecting onto the larger set is theoretically backed up in~\cite{jain2014iterative}, and also in this paper, its usage within the factorized approach has been shown to obtain practical improvements; however, currently there is no theory for this~\cite{bhojanapalli2016dropping}. 

In addition to the above algorithms, \emph{stochastic} gradient methods for low-rank matrix recovery have also been investigated~\cite{wang2017universal,li2016Nonconvex}. The goal of these methods are to reduce the cost of calculating the full gradient in each iteration which typically requires $\O(np^2)$ operations. For instance, \cite{wang2017universal} has combined the factorized method with SVRG~\cite{johnson2013accelerating}, while the authors in~\cite{li2016Nonconvex} have used the SVP algorithm along with SVRG or SAGA~\cite{defazio2014saga} algorithms. However, these algorithms suffer from either heavy computational cost due to the initialization and projection step, or assume stringent conditions on the RSC/RSS conditions. Similar to the factorized method proposed in~\cite{tu2016low}, the method in~\cite{wang2017universal} requires multiple SVDs for the initialization step, and its total running time depends the condition number of the ground truth matrix. In addition, to establish the linear convergence, one needs no limitations on the RSC/RSS conditions. On the other hand, the method in~\cite{li2016Nonconvex} is robust to ill-condition problem and it uses the idea of projection on the set of matrices with larger rank than the true one. However, each iteration of it needs SVD and it may overshadow the benefit of it in alleviating the computation of the gradient. 

Finally, we mention a non-iterative algorithm for recovery of low-rank matrices from a set of nonlinear measurements proposed by~\cite{plan2017high}. While this approach does not need to know the nonlinearity of the link function, its recovery performance is limited, and we can only recover the solution of the optimization problem up to a scalar ambiguity. 

All the aforementioned algorithms suffer from one (or more) of the following issues: expensive computational complexity, slow convergence rate, and troublesome dependency on the condition number of the optimum. In this paper, we resolve these problems by a renewed analysis of approximate low-rank projection algorithms, and integrate this analysis to obtain a new algorithmic framework for optimizing general convex loss functions with rank constraints.

%% file: model.tex
\section{Algorithm and Analysis}
\label{sec:model}
In this section, we propose our algorithm and provide the theoretical results to support it. Before that we introduce some notations and definitions. 

\subsection{Preliminaries}
\label{prelim}
We denote the minimum and maximum eigenvalues of matrix $\bS$ by $S_p$ and $S_1$, respectively. We use $\|A\|_2$ and $\|A\|_F$ for spectral norm and Frobenius norm of a matrix $A$, respectively. 
%Also, $A_r$ denotes the best rank-$r$ approximation (in Frobenius norm) of a given matrix $A$. 
We show the maximum and minimum eigenvalues of a matrix $A\in\R^{p\times p}$ as $\lambda_1(A), \lambda_p(A)$, respectively. In addition, for any subspace $W \subset \mathbb{R}^{p \times p}$, we denote $\P_{W}$ as the orthogonal projection operator onto it. Finally, the phrase ``with high probability" indicates
an event whose failure rate is exponentially small. Our analysis will rely on the following definition~\cite{negahban2009unified,jain2014iterative}:
 
\begin{definition} \label{defRSCRSS}
A function $f$ satisfies the Restricted Strong Convexity (RSC) and Restricted Strong Smoothness (RSS) conditions if for all $L_1,L_2\in\mathbb{R}^{p\times p}$ such that $\rank(L_1)\leq r, \rank(L_2)\leq r$, we have:
\begin{align}
\label{rscrss}
\frac{m_{2r}}{2}&\|L_2-L_1\|^2_F\leq f(L_2) - f(L_1)- \langle\nabla f(L_1) , L_2-L_1\rangle\leq\frac{M_{2r}}{2}\|L_2-L_1\|^2_F, 
\end{align} 
where $m_{2r}$ and $M_{2r}$ are called the RSC and RSS constants, respectively.
\end{definition}

Let $\U_r$ as the set of all rank-$r$ matrix subspaces, i.e., subspaces of $\mathbb{R}^{p \times p}$ that are spanned by any $r$ atoms of the form $uv^T$ where $u, v \in \mathbb{R}^p$ are unit-norm vectors. We will exclusively focus on low-rank approximation algorithms that satisfy the following two properties:
%We will also employ the idea of \emph{head} and \emph{tail} projection introduced by~\cite{hegde2015approximation}, and instantiated in the context of low-rank approximation by~\cite{HegdeFastUnionNips2016}.
\begin{definition}[Approximate tail projection]
\label{taildef}
Let $\epsilon>0$. Then, $\Ta:\R^{p\times p}\rightarrow\U_{r}$ is an approximate tail projection algorithm if for all $L\in \mathbb{R}^{p\times p}$, $\Ta$ returns a subspace $Z=\Ta(L)$ that satisfies: 
$$\|L -\P_Z L\|_F\leq (1+\epsilon)\|L-L_r\|_F,$$ 
where $\P_Z L = ZZ^TL$, and $L_r$ is the optimal rank-$r$ approximation of $L$ in the Frobenius norm.
\end{definition}

\begin{definition}[Per-vector approximation guarantee]
Let $L\in\mathbb{R}^{p\times p}$. Suppose there is an algorithm that satisfies approximate tail projection such that it returns a subspace $Z$ with basis vectors $z_1,z_2,\ldots,z_r$ and approximate ratio $\epsilon$. Then, this algorithm additionally satisfies the per-vector approximation guarantee if 
$$
%\label{pervectorgu}
|u_i^TLL^Tu_i - z_iLL^Tz_i|\leq\epsilon\sigma_{r+1}^2,
$$
where $u_i$'s are the eigenvectors of $L$.
\end{definition}

In this paper, we focus on the randomized Block Krylov SVD (BKSVD) method for implementation of $\Ta$. This algorithm has been proposed by~\cite{musco2015randomized} which satisfies both of these properties with probability at least $99/100$. However, one can alternately use a recent algorithm called LazySVD~\cite{allen2016lazysvd} with very similar properties. For constant approximation ratio $\epsilon$, the asymptotic running time of these algorithms is given by $\widetilde{\O}(p^2 r)$, \emph{independent} of any spectral properties of the input matrix; however, BKSVD ensures a slightly stronger \emph{per-vector} approximation guarantee. 

\begin{algorithm}[!t]
\caption{MAPLE}
\label{alg:appsvp}
\begin{algorithmic}
\STATE \textbf{Inputs:} rank $r$, step size $\eta$, approximate tail projection $\Ta$%, and projection operator onto $\C$
\STATE \textbf{Outputs:} Estimates  $\widehat{L}$
\STATE\textbf{Initialization:} $L^0\leftarrow 0$, $t \leftarrow 0$
\WHILE{$t\leq T$}
\STATE ${L}^{t+1} = \Ta\left(L^{t} - \eta\nabla F(L^t)\right)$ 
\STATE$t\leftarrow t+1$
\ENDWHILE
\STATE\textbf{Return:} $\widehat{L} = L^{T}$
\end{algorithmic}
\end{algorithm}

As we discussed above, our goal is to solve the optimization problem~\eqref{opt_prob}. The traditional approach is to perform projected gradient descent:
$$
{L}^{t+1} = P_r\left(L^{t} - \eta\nabla F(L^t)\right),
$$
where $P_r$ denotes an exact projection onto the space of rank-$r$ matrices, and can be accomplished via SVD. However, for large $p$, this incurs cubic running time and can be very challenging. To alleviate this issue, one can instead attempt to replace the full SVD in each iteration with a tail-{approximate} low-rank projection; it is known that such projections can computed in $O(p^2 \log p)$ time~\cite{clarksonwoodruff}. 

This is precisely our proposed algorithm, which we call {\emph{Matrix Approximation for Low-rank Estimation} (MAPLE)}, is described in pseudocode form as Algorithm~\ref{alg:appsvp}. This algorithm is structurally very similar to~\cite{jain2014iterative,becker2013randomized}. However, the mechanism of~\cite{jain2014iterative} requires \emph{exact} low-rank projections, and~\cite{becker2013randomized} is specific to the least-squares loss function and with weak guarantees. 
%In addition, very recent work of~\cite{HegdeFastUnionNips2016} has proposed approximate subspace-IHT algorithm which uses two step projection for general model of union subspaces, and is limited for squared loss function. 

Here, we show that for low-rank matrix estimation, an \emph{coarse, approximate} low-rank projection (the $\Ta$ operator in Algorithm~\ref{alg:appsvp}) is sufficient for estimating the solution of~\eqref{opt_prob}. A key point is that our algorithm uses approximate low-rank projections with parameter $r$ such that $r\geq r^*$. As we show in Theorem~\ref{AppSVD}, the combination of using approximate projection, together with choosing a large enough rank parameter $r$, enables efficient solution of problems of the form \eqref{opt_prob} for \emph{any (given) restricted convexity/smoothness constants} $M,m$.

Specifically, this ability removes any upper bound assumptions on the ration $\frac{M}{m}$, which have appeared in several recent related works, such as~\cite{bhojanapalli2016dropping}. While the output matrix of {MAPLE} may have larger rank than $r^*$, one can easily post-process it with an final hard thresholding step in order to enforce the result to have exactly rank $r^*$.

In Algorithm~\ref{alg:appsvp}, the choice of approximate low-rank projections is flexible, as long as the approximate tail and per-vector approximation guarantee are satisfied. We note that tail-approximate low-rank projection algorithms are widespread in the literature~\cite{clarksonwoodruff_old,drineas_mahoney,tygert}; however, per-vector approximation guarantee algorithms are less common. As will become clear in the proof of Theorem~\ref{AppSVD}, the per-vector guarantee is crucial in our analysis. 

In our implementation of MAPLE, we invoke the BKSVD method for low-rank approximation mentioned above\footnote{We note that since the BKSVD algorithm is randomized while the definitions of approximate tail projection and per-vector approximation guarantee are deterministic. Fortunately, the running time of BKSVD depends only logarithmically on the failure probability, and therefore an additional union bound argument is required to precisely prove algorithmic correctness of our method.}. Assuming BKSVD as the approximate low-rank projection of choice, %which satisfy two above properties,
we now prove a key structural result about the non-expansiveness of $\Ta$. This result, to the best of our knowledge, is novel and generalizes a recent result reported in~\cite{shen2016tight,li2016Nonconvex}. (We defer the full proof of all theoretical results to the appendix.)

\begin{lemma}~\label{ApprHTh}
For $r>(1+\frac{1}{1-\epsilon})r^*$ and for any matrices $L, L^* \in\R^{p\times p}$ with $\rank(L^*)=r^*$, we have
\begin{align*}
\|\Ta(L)-L^*\|_F^2\leq\left(1+\frac{2}{\sqrt{1-\epsilon}}\frac{\sqrt{r^*}}{\sqrt{r-r^*}}\right)\|L-L^*\|_F^2,
\end{align*}
where $\Ta:\R^{p\times p}\rightarrow\U_{r}$ denotes the approximate tail projection defined in Definition~\ref{taildef} and $\epsilon>0$ is the corresponding approximation ratio.
\end{lemma}

\begin{proof}[proof sketch]
The proof follows the approach of~\cite{li2016Nonconvex} where it is first given for sparse hard thresholding, and then is generalized to the low-rank case using Von Neumann's trace inequality. First, define $\theta = [\sigma_1^2(L),\sigma_2^2(L)\ldots,\sigma_r^2(L)]^T$. Also let $\theta^* = [\sigma_1^2(L^*),\sigma_2^2(L^*)\ldots,\sigma_{r^*}^2(L^*)]^T$, and $\theta' = \Ta(\theta)$. Also, let $supp(\theta^*)  = \mathcal{I^*}$, $supp(\theta)  = \mathcal{I}$, $supp(\theta')  = \mathcal{I'}$, and $\theta'' =\theta - \theta'$ with support $I''$. 

Now define new sets $\mathcal{I^*}\cap\mathcal{I'} = \mathcal{I}^{*1}$ and $\mathcal{I^*}\cap\mathcal{I''} = \mathcal{I}^{*2}$ with restricted vectors to these sets as $\theta_{\mathcal{I}^{*1}} = \theta^{*1}$, $\theta_{\mathcal{I}^{*2}} = \theta^{*2}$, $\theta'_{\mathcal{I}^{*1}} = \theta^{1*}$, $\theta''_{\mathcal{I}^{*2}} = \theta^{2*}$ such that $|\mathcal{I}^{*2}| = r^{**}$, and $\theta_{\max} = \|\theta^{2*}\|_{\infty}$.

The proof continues by upper bounding the ratio of $\frac{\|\theta'-\theta^*\|_2^2 - \|\theta- \theta^*\|_2^2}{\|\theta- \theta^*\|_2^2}$ in terms of $r,r^*, r^{**}$ and by using the inequality $\widehat{\theta}_{\min}\geq(1-\epsilon)\theta_{\max}$ where $\widehat{\theta}$ denotes the vector of approximate eigenvalues returned back by $\Ta$. This inequality is resulted by invoking the per-vector guarantee property of $\Ta$. We can now obtain the desired upper bound to get the final claim.
\end{proof} 

We now leverage the above lemma to provide our main theoretical result supporting the algorithmic efficiency of MAPLE.

\begin{theorem}[Linear convergence of MAPLE]
\label{AppSVD}
Assume that the objective function $F(L)$ satisfies the RSC/RSS conditions with parameters $M_{2r+r^*}$ and $m_{2r+r^*}$. Define $\nu = \sqrt{1+\frac{2}{\sqrt{1-\epsilon}}\frac{\sqrt{r^*}}{\sqrt{r-r^*}}}$. Let $J_t$ denote the subspace formed by the span of the column spaces of the matrices $L^t, L^{t+1}$, and $L^*$, the solution of~\eqref{opt_prob}. In addition, assume that $r >\frac{C_1}{1-\epsilon}\left(\frac{M_{2r+r^*}}{m_{2r+r^*}}\right)^4r^*$ for some $C_1>2$. Choose step size as $\eta$ as $\frac{1-\sqrt{\alpha'}}{M_{2r+r^*}}\leq\eta\leq\frac{1+\sqrt{\alpha'}}{m_{2r+r^*}}$ where $\alpha' = \frac{\sqrt{\alpha-1}}{\sqrt{1-\epsilon}\sqrt{\alpha-1}+2}$ {for some $\alpha = \Theta(r/r^*) > 1$}. Then, MAPLE outputs a sequence of estimates $L^t$ such that:
\begin{align}
\label{linconvApp}
\|L^{t+1} - L^{*}\|_F\leq \rho\|L^{t} - L^{*}\|_F + \nu\eta\|\P_{J_t}\nabla F(L^*)\|_F,
\end{align}
where $\rho = \nu\sqrt{1+M_{2r+r^*}^2\eta^2 - 2m_{2r+r^*}\eta} < 1$.  
\end{theorem}

We have to mention that $L_*$ can be any rank $r$ matrix which of course includes the solution of~\eqref{opt_prob}. Also, Theorem~\ref{AppSVD} guarantees the linear convergence of the MAPLE algorithm up to a given radius of convergence determined by the gradient of $F$ at $L^*$. We note that the contraction factor $\rho$ is not affected by extent to which the objective function $F(L)$ is strongly smooth/convex. In other words, no matter how large the ratio $\frac{M}{m}$ is,  its effect is balanced by $\nu$ through choosing large enough $r$. Also, the quality of the estimates in Theorem~\ref{AppSVD} is upper-bounded by the gradient term $\|\P_{J_t}\nabla F(L^{*})\|_F$ in~\eqref{linconvApp}, within each iteration. 

Below, we instantiate the general optimization problem~\eqref{opt_prob} in the context of three estimation problems (NLARM, logistic PCA, and PME). In NLARM and PME, $L^*$ denotes the ground truth which we are looking for to estimate; as a result, the gradient term in~\eqref{linconvApp} represents the statistical aspect of MAPLE. For these problems, we give an upper bound on this term. Also, we show that the loss function $F(L)$ satisfies the RSC/RSS conditions in all three instantiations, and consequently, derive the sample complexity and the running time of MAPLE. 

\subsection{Discussion: Main Results and Novelty}

First we note that the randomized SVD approach being used in \cite{becker2013randomized} and MAPLE are algorithmically the same. However, the algorithm in \cite{becker2013randomized} has been analyzed only for squared loss, and its theoretical guarantees are somewhat weak. Specifically, in the regime of parameters required to obtain optimal sample complexity (i.e., $n=\O(pr)$), the quality of their approximate projection should be the order of $\O(1/p)$ in order to have provable convergence. This inflates the running time to cubic, and overshadows the usage of any approximate SVD methods. On the other hand, MAPLE can handle general loss functions that satisfy the RSC/RSS conditions. Specifically, the analysis in MAPLE exploits the novel structural result for approximate rank-$r$ projection onto the space of rank-$r$ matrices ($r\gg r^*$) (Lemma~\ref{ApprHTh}), which shows that each projection step in MAPLE is \emph{nearly non-expansive}. This is a crucial new theoretical result in our paper, and is a geometric property of any partial SVD routine which satisfies a per-vector approximation guarantee (and this can be of independent interest in other low-rank estimation problems as well). 

Second, for approximate tail projection, MAPLE uses a \emph{gap-independent} SVD method which guarantees that the running time for calculating the approximation of right singular vectors takes $\Otilde(p^2r)$ operations in each iteration. This step is crucial as even projection onto a subspace with rank-$1$ can take $\O(p^3)$ time due to the existence of a vary small gap between $r^{th}$ and $(r+1)^{th}$ singular values~\cite{musco2015randomized}. Here, one might ask that the classical methods are better than the gap independent result in~\cite{musco2015randomized} if the approximate ratio, $\epsilon$ is less than the $gap$. However, this is not the case in our setup, since we do not need to be very accurate in computing the approximation of right singular values (achieving very small $\epsilon$). Indeed $\epsilon$ is given by $iter = \Theta(\frac{\log p}{\sqrt{\epsilon}})$ where $iter$ denotes the number of iterations required in BKSVD. In all our experiments, we have chosen $iter=2$ which implies very large $\epsilon$ close to $1$ is sufficient for tail projection. On the other hand, the spectral gap can be a very small number, i.e., $10^{-6}$ for many matrices encountered in practice. %As a result, in real applications, BKSVD results a gap independent approach. 

Finally, we highlight the ability of MAPLE for handling the objective functions $F(L)$ with arbitrary large smoothness-to-convexity ratio $\frac{M}{m}$. For functions even with very large condition number $\frac{M}{m}$, MAPLE has the ability to choose a projected rank $r\gg r^*$ to guarantee the convergence. This is the role of $\nu$ in the expression of the contraction factor, $\rho$ in~\eqref{linconvApp}; no matter how large $\frac{M}{m}$ is, its effect is balanced by $\nu$. To see this, fix $m$, and let $M$ be a given arbitrary large value, then by choosing $r>M^4r^*$, and step size as stated in the theorem, we can guarantee that $\rho<1$; hence, establishing linear  convergence. 

We note that a good choice of step size (which is constant) does depend on problem parameters, as is the case for many other first order algorithms. In practice, this has to be appropriately tuned. However, assuming this choice is made, the convergence rate is not affected.

\section{Applications}

We now instantiate the MAPLE framework in three low-rank matrix estimation problems of interest.

\subsection{Nonlinear Affine Rank Minimization}
\label{NLARMsec}

\begin{table}
\caption{Summary of our contributions, and comparison with existing methods for NLARM.
%\propto\frac{1+\delta}{1-\delta}$ where $\delta$ is the RIP constant
 $\kappa$ denotes the condition number of $L^*$, and $\vartheta$ denotes the final optimization error. Also, SC and RT denote sample complexity and running time, respectively. Here we have presented (for each algorithm) the best available running time result.}
\label{NLARMtable}
\begin{center}
\begin{small}
\renewcommand{\arraystretch}{1.5}
\begin{tabular}{lccr}
\hline
Algorithm & SC & RT &  Bounded $\frac{M}{m}$\\
\hline
{Convex~\cite{recht2010guaranteed}} &$\Otilde(pr^*)$  & $\O(\frac{p^3}{\sqrt{\vartheta}})$ & Yes   \\
Non-convex Reg~\cite{quanming2017large} & $\Otilde(pr^*)$ & $\O(\frac{p^2r^*}{\vartheta})$ & Yes \\
Factorized~\cite{bhojanapalli2016dropping}  & $\Otilde(pr^*)$& $\O(p^2(r^*+\log p)\kappa^2\log(\frac{1}{\vartheta}) + p^3)$ & Yes \\
SVP~\cite{jain2014iterative} & $\Otilde(pr^*)$ & $\O(p^3\log(\frac{1}{\vartheta})) $& No  \\
%Stochastic~\cite{Qoun}& $r^*$ & $\O((np^2 + \kappa^2bp^2)\log(\frac{1}{\vartheta})+p^3)$ & Yes \\
\textbf{MAPLE}  &$\mathbf{{\Otilde(pr^*)}}$  &  $\mathbf{\O(p^2r^*\log p\log(\frac{1}{\vartheta}))}$ & {No} \\
\hline
\end{tabular}
\end{small}
\end{center}
\end{table}

Consider the nonlinear observation model $y = g(\A(L^*)) + e$, where $\A$ is a linear operator, $\A:\R^{p\times p}\rightarrow\R^n$ parametrized by $n$ full rank matrices, $A_i\in\R^{p\times p}$ such that $(\A(L^*))_i = \langle A_i,L^*\rangle$ for $i=1,\ldots,n$. Also, $e$ denotes an additive subgaussian noise vector with i.i.d., zero-mean entries that is also assumed to be independent of $\A$ (see appendix for more details). If $g(x)=x$, we have the well-known matrix sensing problem for which a large number of algorithms have been proposed. The goal is to estimate the ground truth matrix $L^*\in\R^{p\times p}$ for more general nonlinear link functions.

We assume that link function $g(x)$ is a differentiable monotonic function, satisfying $0<\mu_1\leq g'(x)\leq\mu_2$ for all $x\in\mathcal{D}(g)$ (domain of $g$). This assumption is standard in statistical learning~\cite{kakade2011} and in nonlinear sparse recovery~\cite{negahban2009unified,yang2015sparse,soltani2016fastIEEETSP17}. Also, as we will discuss below, this assumption will be helpful for verifying the RSC/RSS condition for the loss function that we define as follows. We estimate $L^*$ by solving the optimization problem:
\begin{equation} \label{opt_probNLARM} 
\begin{aligned}
& \underset{L}{\text{min}}
\quad F(L) = \frac{1}{n}\sum_{i=1}^{n}\Omega(\langle A_i,L\rangle) - y_i\langle A_i,L\rangle\\
& \text{s.t.}\ \ \quad\rank(L)\leq r^*, 
\end{aligned}
\end{equation}
where $\Omega : \mathbb{R} \rightarrow \mathbb{R}$ is chosen such that $\Omega'(x) = g(x)$.\footnote{The objective functioon $F(L)$ in~\eqref{opt_probNLARM} is standard; see~\cite{soltani2016fastIEEETSP17} for an in-depth discussion.} 
Due assumption on the derivative of $g$, we see that $F(L)$ is a convex function (actually strongly convex), and can be considered as a special case of general problem in~\eqref{opt_prob}. 

We assume that the design matrices $A_i$'s are constructed as follows. Consider a partial Fourier or partial Hadamard matrix $X'\in\R^{n\times p^2}$ which is multiplied from the right by a diagonal matrix, $D$, whose diagonal entries are uniformly distributed over $\{-1,+1\}^{p^2}$. Call the resulting matrix $X = X'D$ where each row is denoted by $X_i^T\in\R^{p^2}$. If we reshape each of these rows as a matrix, we obtain ``measurement'' (or ``design'') matrices $A_i\in\R^{p\times p}$ for $i=1,\ldots,n$. This particular choice of design matrices $A_i$'s is because they support fast matrix-vector multiplication which takes $\O(p^2\log(p))$. (The origins of constructing design matrices of this form come from the compressive sensing literature~\cite{modelcsICALP}). %If we use Gaussian ensembles for $A_i$, we cannot take advantage of our approximate algorithm and also factorized methods in terms of running time due to the required $\O(p^3)$ operations needing for calculating $\A(L)$.  

The following theorem gives the upper bound on the term, $\|\P_{J_t}\nabla F(L^{*})\|_F$, that appears in Theorem~\ref{AppSVD}. This can be viewed as a ``statistical error'' term, and is zero in the absence of noise.
\begin{theorem}\label{staterrorNLARM}
Consider the observation model $y = g(\A (L^*)) + e$ as described above. Let the number of samples scale as $n= \O(pr~\textrm{polylog}~(p))$, then with high probability, for any given subspace $J$ of $\mathbb{R}^{p \times p}$, 
%at least $1-\exp(-cn\varpi^2)$ ($c>0$ is a constant),   \frac{1}{\varpi^2}
we have for $t=1,\ldots,T$:
\begin{align}
\|\P_{J_t}\nabla F(L^{*})\|_F\leq\frac{1+\delta_{2r+r^*}}{\sqrt{n}}\|e\|_2,
\end{align}
where $0<\delta_{2r+r^*}<1$ denotes the RIP constant of $\A$.
\end{theorem}

\begin{corollary}\label{induclin}
Consider all the assumptions and definitions stated in Theorem~\ref{AppSVD}. If we initialize MAPLE with $L^0 = 0$, then after $T_{iter} = \mathcal{O}\left(\log\left(\frac{\|L^*\|_F}{\vartheta}\right)\right)$ iterations, we obtain:
\begin{align}
\|L^{T+1} - L^{*}\|_F\leq \vartheta + \frac{1}{\sqrt{n}}\frac{\nu\eta(1+\delta_{2r+r^*})}{1-\rho}\|e\|_2,
\end{align}
for some $\vartheta>0$.
\end{corollary}
We now provide conditions under which the RSS/RSC assumptions in Theorem~\ref{AppSVD} are satisfied.
\begin{theorem}[RSC/RSS conditions for MAPLE] 
Let the number of samples scale as $n= \O(pr~\mathrm{polylog}~(p))$. Assume that $\frac{\mu_2^4(1+\omega)^4}{\mu_1^4(1-\omega)^4}\leq C_2(1-\epsilon)\frac{r}{r^*}$ for some $C_2, \omega>0$ and $\epsilon>0$ denotes the approximation ratio in Algorithm~\ref{alg:appsvp}. Then with high probability, the  loss function $F(L)$ in~\eqref{opt_probNLARM} satisfies the RSC/RSS conditions with constants $m_{2r+r^*}\geq \mu_1(1-\omega)$ and $M_{2r+r^*}\leq \mu_2(1+\omega)$
in each iteration.
\label{RSCRSSappNLARM}
\end{theorem}

\textbf{Sample complexity.} By Corollary~\ref{induclin} and Theorem~\ref{RSCRSSappNLARM}, the sample complexity of MAPLE algorithm is given by $n= \O(pr~\mathrm{polylog}~(p))$ in order to achieve a specified estimation error. This sample complexity is nearly as good as the optimal rate, $\O(pr)$. We note that the leading constant hidden within the $\O$-notation depends on $\rho, \eta$, the RIP constant of the linear operator $\A$, and the magnitude of the additive noise. (Since we assume that this noise term is subgaussian, it is easy to show that $\|e\|_2$ scales as $\O(\sqrt{n})$ in expectation and with high probability).

\textbf{Time complexity.} Each iteration of MAPLE needs to compute the gradient, plus an approximate tail projection to produce a rank-$r$ matrix. Computing the gradient involves one application of the linear operator $\A$ for calculating $\A(L)$, and one application of the adjoint operator, i.e., $\A^*(y-g(\A(L))$. Let $T_{mult}$ and $T'_{mult}$ denote the required time for these operations, respectively. On the other hand, approximate tail projection takes $\O\left(\frac{p^2r\log p}{\sqrt{\varepsilon}}\right)$ operations for achieving the approximate ratio $\epsilon$ according to~\cite{musco2015randomized}. Thanks to the linear convergence of MAPLE,  the total number of iterations for achieving $\vartheta$ accuracy is given by $T_{iter} = \mathcal{O}\left(\log\left(\frac{\|L^*\|_F}{\vartheta}\right)\right)$. Let $\pi = \frac{M}{m}$; thus, the overall running time scales as $T = {\O}\left(\left(T_{mult} +T'_{mult}+ \frac{p^2r^*\pi^4\log p}{\sqrt{\epsilon}}\right)\left(\log\frac{\|L^*\|_F}{\vartheta}\right)\right)$ by the choice of $r$ according to Theorem~\ref{AppSVD}. If we assume that the design matrices $A_i$'s are implemented via a Fast Fourier Transform, computing $T_{mult} =T'_{mult}$ takes ${\O}(p^2\log p)$ operations. As a result, $T = \O\left(\left(p^2\log p +\frac{p^2r^*\pi^4\log p}{\sqrt{\epsilon}}\right)\left(\log\frac{\|L^*\|_F}{\vartheta}\right) \right)$. 

In Table~\ref{NLARMtable}, for $g(x) = x$ and the linear operator $\A$ defined above, we summarize the sample complexity as well as (asymptotic) running time of several algorithms. {In this table, we assume a constant ratio of $M/m$ for all the algorithms.} We find that all previous methods, while providing excellent sample complexity benefits, suffer from either cubic dependence on $p$, or inverse dependence on the estimation error $\vartheta$, or quadratic dependence on the condition number $\kappa$ of the ground truth matrix. In contrast, MAPLE enjoys (unconditional) $\Otilde(p^2 r^*)$ dependence, which is (nearly) linear in the size of the matrix for small enough $r^*$.     

\subsection{Logistic PCA}
\label{LPCA}
Principle component analysis (PCA) is a widely used statistical tool in various applications such as dimensionality reduction, denoting, and visualization, to name a few. While the regular PCA sometimes called linear PCA can be applied for any data type, its usage for binary or categorical observed data is not satisfactory, due to the fact that it tries to minimize a least square objective function. %In other words, PCA is useful where the likelihood of underlying data is distributed as Gaussian. 
As a result, applying it to the binary case makes the result less interpretable~\cite{jolliffe1986principal}. 

To alleviate this issue, one can assume that each row of the observed binary matrix (a sample data) follows the multivariate Bernoulli distribution such that its maximum variations can be captured by a low-dimensional subspace, and then use the logistic loss to find this low-dimensional representation of the observed data. This problem has been also studied in the context of collaborative filtering on binary data~\cite{johnson2014logistic}, one-bit matrix completion~\cite{davenport20141}, and network sign prediction~\cite{chiang2014prediction}. 
Mathematically, consider an observed binary matrix $Y\in\R^{p\times p}$ with entries belong to set ${\{}0, 1{\}}$ such that the mean of each $Y_{ij}$ is given by $p_{ij} = P(Y_{ij} = 1|L_{ij}) = \sigma(L^*_{ij})$ where $\sigma(z)  = \frac{1}{1+\exp(-z)}$. The goal is to estimate a low-rank matrix $L^*$ such that $L^*_{ij} = \log(\frac{p_{ij}}{1-p_{ij}}) = \text{logit}(p_{ij})$ by minimizing the following regularized logistic loss:
\begin{equation}\label{logisticPCA}
\begin{aligned}
&\underset{L}{\text{min}} 
\quad F(L) = -\sum_{i,j}\big{(}Y_{ij}\log\sigma(L_{ij}) + (1-Y_{ij})\log(1-\sigma(L_{ij}))\big{)} + \lambda\|L\|_F^2\\
& \text{s.t.}\ \ \quad\rank(L)\leq r^*,
\end{aligned}
\end{equation}
where $\lambda>0$ is a tuning parameter.\footnote{Here, $L_*$ denotes the solution of the optimization problem~\eqref{logisticPCA}.} We note that the objective function in~\eqref{logisticPCA} without the regularizer term is only strongly smooth. By adding the Frobenius norm of the optimization variable, we ensure that it is also globally strongly convex (Hence, RSC/RSC conditions are trivially satisfied). Here, we focus on finding the solution of~\eqref{logisticPCA}, $L^*$. Hence, we do not have explicitly the notion of ground truth as previous application.

For solving the optimization problem~\eqref{logisticPCA}, several algorithms have been proposed in recent years. Unfortunately, algorithms such as convex nuclear norm minimization are either too slow, or do not have theoretical guarantees~\cite{chiang2014prediction,johnson2014logistic,davenport20141}. Very recently, a non-convex factorized algorithm proposed by~\cite{park2016finding} has been supported by rigorous convergence analysis. We will compare the performance of this algorithm with the MAPLE in the experimental section. In particular, we show that the running time of MAPLE for solving the above problem is given by $\Otilde(rp^2)$ as the dominating term is related to the projection step and gradient calculation takes $\O(p^2)$ time.  

\subsection{Precision Matrix Estimation (PME)}
\label{PME}
{Gaussian graphical models} are a popular tool for modeling the interaction of a collection of Gaussian random variables. In Gaussian graphical models, nodes represent random variables and edges model conditional (in)dependence among the variables
~\cite{wainwright2008graphical}. Over the last decade, significant efforts have been directed towards algorithms for learning \emph{sparse} graphical models. 

Mathematically, let $\Sigma^*$ denote the positive definite covariance matrix of $p$ Gaussian random variables, and let $\Theta^* = (\Sigma^*)^{-1}$ be the corresponding precision matrix. Then, $\Theta^*_{ij} = 0$ implies that the $i^{\textrm{th}}$ and $j^{\textrm{th}}$ variables are conditionally independent given all other variables and the edge $(i,j)$ does not exist in the underlying graph. The basic modeling assumption is that $\Theta^*$ is sparse, i.e., such graphs possess only a few edges. Such models have been fruitfully used in several applications including astrophysics~\cite{padmanabhan}, scene recognition~\cite{souly}, and genomic analysis~\cite{yin2013}. Numerous algorithms for sparse graphical model learning -- both statistically as well as computationally efficient -- have been proposed in the machine learning literature~\cite{friedman2008sparse,mazumder2012graphical,banerjee2008model,hsieh2011sparse}. 
Unfortunately, sparsity is a simplistic first-order model and is not amenable to modeling more complex interactions. For instance, in certain scenarios, only some of the random variables are directly observed, and there could be relevant \textit{latent} interactions to which we do not directly have access.  

The existence of latent variables poses a significant challenge in graphical model learning since they can confound an otherwise sparse graphical model with a dense one. This scenario is illustrated in Figure~\ref{lvgfig}. Here, nodes with solid circles denote the observed variables, and solid black edges are the ``true" edges. One can see that the ``true" graph is rather sparse. However, if there is even a single unobserved (hidden) variable denoted by the node with the broken red circle, then it will induce dense, apparent interactions between nodes that are otherwise disconnected; these are denoted by the dotted black lines. A flexible and elegant method to learn latent variables in graphical models was proposed by~\cite{chandrasekaran2012latent}. At its core, the method imposes a {superposition} structure in the observed precision matrix as the sum of \emph{sparse} and \emph{low-rank} matrices, i.e.,
$ \Theta^* = S^{*} + L^{*} $.
Here, $\Theta^*, S^*, L^*$ are $p \times p$ matrices where $p$ is the number of variables. The matrix $S^*$ specifies the conditional observed precision matrix given the latent variables, while $L^*$ encodes the effect of marginalization over the latent variables. The rank of $L^*$, $r^*$, is equal to the number of latent variables and we assume that $r^*$ is much smaller than $p$. The goal is to estimate precision matrix $\Theta^*$. Here, we merely focus on the learning the low-rank part, and assume that the sparse part is a known prior.\footnote{For, instance, if the data obeys the spiked covariance model~\cite{johnstone2001distribution}, the covariance matrix is expressed as the sum of a low-rank matrix and a diagonal matrix. Consequently, by the Woodbury matrix identity, the precision matrix is the sum of a diagonal matrix and a low-rank matrix; $\Theta^* = \bS + L^*$. In addition, problem in~\eqref{opt_probPME} is similar to the latent variable in Gaussian graphical model proposed by~\cite{Venkat2010latent}.}

%To learn such a superposition model,~\cite{chandrasekaran2012latent} propose a regularized maximum-likelihood estimation framework, with $\ell_1$-norm and nuclear norm penalties as regularizers; these correspond to \emph{convex} relaxations of the sparsity and rank constraints, respectively. Using this framework, they prove that such graphical models can be learned with merely $n = O(pr)$ random samples. However, using this convex relaxation is very slow in high-dimensional regime. Several subsequent works~\cite{ma2013alternating,hsieh2014quic} have attempted to provide faster algorithms, but all known (provable) methods involve at least \emph{cubic} worst-case running time. 

We cast the estimation of matrix $\Theta^*$ into our framework. Suppose that we observe samples $x_1,x_2,\ldots,x_n \overset{i.i.d}{\thicksim}\mathcal{N}(0,\Sigma)$ where each $x_i\in\mathbb{R}^{p}$. Let $C= \frac{1}{n}\sum_{i=1}^{n}x_ix_i^{T}$ denote the sample covariance matrix, and $\Theta^* = (\Sigma^*)^{-1}$ denote the true precision matrix. Following the formulation of~\cite{han2016fast}, we solve the following minimization of NLL problem:
\begin{equation} \label{opt_probPME} 
\begin{aligned}
& \underset{L}{\text{min}}
& & F(L) = -\log \ \det (\bS+L) + \langle \bS+L,C\rangle\\
& \text{s.t.} 
& & \rank(L)\leq r^*, \ L\succeq 0 .
\end{aligned}
\end{equation}
where $\Theta^* = \bS + L^*$ such that $\bS$ is a known positive diagonal matrix (in general, a positive definite matrix) imposed in the structure of precision matrix to make the above optimization problem well-defined. We will exclusively function in the high-dimensional regime where $n \ll p^2$. As an instantiation of the general problem~\eqref{opt_prob}, our goal is to learn the low-rank matrix $L^{*}$ with rank $r^*\ll p$, from samples $x_i$'s. We provide a summary of the theoretical properties of our methods, and contrasts them with other existing methods for PME existing methods in Table~\ref{ComptablePME} (We assume a constant ratio of $M/m$ for all the algorithms).

%While problem~\eqref{opt_probPME} has extra PSD cone constraint compared to the general problem~\eqref{opt_prob}, we can still use MAPLE. Please see Theorem~\ref{mineigval}. 

\begin{table*}
\caption{Summary of our contributions, and comparison with existing methods. Here, $\gamma = \sqrt{\frac{\sigma_r}{\sigma_{r+1}}-1}$ represents the spectral gap parameter. } %The running time of the ADMM approach is marked as $\poly(p)$ since the precise rate of convergence is unknown.
\label{ComptablePME}
%\vskip 0.1in
\begin{center}
\begin{small}
\renewcommand{\arraystretch}{1.5}
\begin{tabular}{lccr}
\hline
%\abovespace\belowspace
%\begin{sc}
Algorithm & Running Time & Spectral dependency \\%& Output rank\\
%\end{sc}
\hline
%\abovespace
SDP \cite{chandrasekaran2012latent}  &$\poly(p)$  & Yes \\%&   $\gg r$   \\
ADMM\cite{ma2013alternating} & $\poly(p)$ & Yes \\%& $\gg r$ \\
QUICDIRTY\cite{yang2013dirty}  & $\Otilde(p^3)$ & Yes \\%& $\gg r$ \\
SVP\cite{jain2014iterative} & $\Otilde(p^3)$ & No \\%& $\gg r$ \\
Factorized\cite{bhojanapalli2016dropping} & $\Otilde(p^2r^* / \gamma)$ & Yes \\%& $r$ \\
%Exact Projection & $\mathbf{\Otilde\left(p^3 \right)}$ & $\mathbf{No}$ \\%& $\mathbf{r}$ \\
\textbf{MAPLE}&$\mathbf{\Otilde(p^2 r^*)}$  &  $\mathbf{No}$ \\%& $\mathbf{r}$ \\
\hline
\end{tabular}
%\end{sc}
\end{small}
\end{center}
%\vskip -0.1in
\end{table*}

As an illustration of our results, we first analyze the \textbf{exact projected-gradient approach}, which is a slight variant of the approach of~\cite{li2016Nonconvex}, since its analysis for establishing RSC/RSS is somewhat different from ours. In this setup, the algorithm starts with a zero initialization and proceeds in each iteration as $L^{t+1} = \mathcal{P}_r^+\left(L^{t} - \eta'\nabla F(L^t)\right)$ where $\P_r^+(\cdot)$ for some $r>r^*$ denotes projection onto the space of rank-$r$ matrices which is implemented through performing an exact eigenvalue decomposition (EVD) of the input and selecting the nonnegative eigenvalues and corresponding eigenvectors~\cite{henrion2012projection}.\footnote{Note that we may not impose a PSD projection within every iteration. If an application requires a PSD matrix as the output (i.e., if proper learning is desired), then we can simply post-process the final estimate $\widehat{L}$ by retaining the nonnegative eigenvalues (and corresponding eigenvectors) through an exact EVD.}
The following theorem shows an upper bound on the estimation error of the low-rank matrix at
each iteration through exact projected-gradient approach.

\begin{figure}[t]
\begin{center}
\begin{tikzpicture}[ultra thick,every node/.style={minimum size=3.5em},decoration={border,segment length=2mm,amplitude=0.3mm,angle=90},scale=0.37, transform shape]
  \foreach \x in {1,...,7}{%
    \pgfmathparse{(\x-1)*360/8}
    \node[draw,circle] (N-\x) at (\pgfmathresult:5.4cm) [thick] {\Large \pgfmathparse{(\x==2)?"$p$":((\x==1)?"$p-1$":int(\x-2))}\pgfmathresult};
  }
  \pgfmathparse{7*360/8}
  \node[draw,circle,red,inner sep=0.25cm,decorate] (N-8) at (\pgfmathresult:5.4cm) {\Large $h$};
  
  \foreach \x in {1,...,7}{%
    \foreach \y in {\x,...,7}{%
        \path (N-\x) edge[thin, dotted] (N-\y);
  }
  
  \path (N-2) edge[thick,-] (N-5);
  
  \path (N-2) edge[thick,-] (N-7);

  \path (N-1) edge[thick,-] (N-7);
  
  \path (N-2) edge[thick,-] (N-4);

  \path (N-5) edge[thick,-] (N-6);
  
  }
    \foreach \y in {1,...,8}{%
        \path (N-8) edge[red, very thick, dotted] (N-\y);
  }
\end{tikzpicture}
\end{center}
\caption{Illustration of effects of latent variable in graphical model learning. Solid edges represent ``true" conditional dependence, while dotted edges represent apparent dependence due to the presence of the latent variable $h$.
\label{lvgfig}}
\vskip -3mm
\end{figure}
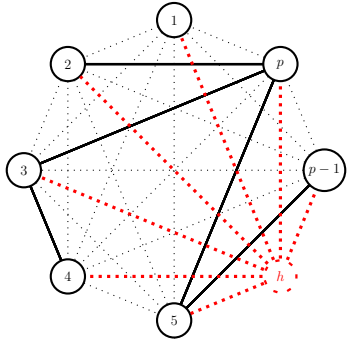

\begin{theorem}[Linear convergence with exact projected-gradient approach]
\label{ExactSVD}
Assume that the objective function $F(L)$ satisfies the RSC/RSS conditions with corresponding constants as $M_{2r+r^*}$ and $m_{2r+r^*}$. Define $\nu' = \sqrt{1+\frac{2\sqrt{r^*}}{\sqrt{r-r^*}}}$. Let $J_t$ denotes the subspace formed by the span of the column spaces of the matrices $L^t, L^{t+1}$, and $L^*$.
% where $L^t$ and $L^*$ denote the estimation of the low-rank matrix by Alg~\ref{alg:svp} at iteration $t$, and the true unknown low-rank matrix, respectively. 
In addition, assume that $r >C_1'\left(\frac{M_{3r}}{m_{3r}}\right)^4r^*$ for some $C'_1>0$. Choose step size $\eta'$ as $\frac{1-\sqrt{\beta'}}{M_{2r+r^*}}\leq\eta'\leq\frac{1+\sqrt{\beta'}}{m_{2r+r^*}}$ where $\beta' = \frac{\sqrt{\beta-1}}{\sqrt{\beta-1}+2}$ for some $\beta >1$. Then, exact projected-gradient outputs a sequence of estimates $L^t$ such that:

\begin{align}
\label{linconvEx}
\|L^{t+1} - L^{*}\|_F\leq \rho'\|L^t-L^*\|_F + \nu'\eta'\|\P_{J_t}\nabla F(L^*)\|_F,
\end{align}
where $\rho' = \nu'\sqrt{1+M_{3r}^2\eta'^2 - 2m_{3r}\eta'}$.  
\end{theorem}

The quality of the estimates in Theorems~\ref{ExactSVD} is upper-bounded by the gradient term $\|\P_{J_t}\nabla F(L^{*})\|_F$ in~\eqref{linconvEx} within each iteration. The following theorem establishes this bound:

\begin{theorem}
\label{BoundGrad}
Under the assumptions of Theorem~\ref{ExactSVD}, for any fixed $t$ we have:
\begin{align}
\label{BGRExc}
\|\P_{J_t}\nabla F(L^{*})\|_F\leq c_2\sqrt{\frac{rp}{n}},
\end{align}
%Similarly, under the assumptions of Theorem~\ref{AppSVD}, 
%\begin{align}
%\label{BGRapp}
%\|\P_{V_t}\nabla F(L^{*})\|_F\leq c_3\sqrt{\frac{rp}{n}},
%\end{align}
with probability at least $1 - 2\exp(-p)$ where $c_2>0$ is an absolute constant.
\end{theorem}

%{We will show later that the second term on the right hand side of this inequality is upper-bounded by an arbitrarily small constant with sufficient number of samples in both of applicaitons. Also, the first term decreases exponentially with iteration count. Overall, after $T = \mathcal{O}\left(\log_{1/\rho}\left(\frac{\|L^*\|_F}{\vartheta}\right)\right)$ iterations, we obtain an upper-bound of $O(\vartheta)$ on the total estimation error, indicating linear convergence.}

Next, we verify the RSS/RSC conditions of the objective function defined in~\eqref{opt_probPME}, justifying the assumptions made in Theorem~\ref{ExactSVD} (please see appendix for full expression of the sample complexity in terms of the leading constants).

\begin{theorem}[RSC/RSS conditions for exact projected-gradient approach]
\label{RSCRSSex}  %n=\O\left(\frac{1}{\delta^2}\left(\frac{\eta'}{1-\rho'}\right)^2rp\right)$ for some small constant $\delta>0$ and $\rho'$ defined as Theorem~\ref{ExactSVD}
Let the number of samples scaled as $n=\O\left(pr\right)$. Also, assume that
$$S_p\leq S_1\leq C_2''(\frac{r}{r^*})^{\frac{1}{8}}S_p - \left(1+\sqrt{r^*}\right)\|L^*\|_2- \delta.$$ 
Then, the loss function $F(L)$ in~\eqref{opt_probPME} satisfies RSC/RSS conditions with constants $m_{2r+r^*}\geq \frac{1}{(S_1 +\left(1+\sqrt{r}\right)\|L^*\|_2+ \delta)^2}$ and $M_{2r+r^*}\leq \frac{1}{S_p^2}$ that satisfy the assumptions of Theorem~\ref{ExactSVD} in each iteration.
\end{theorem}

The above theorem states that convergence of our method is guaranteed when the eigenvalues of $\bS$ are roughly of the same magnitude, and large when compared to the spectral norm of $L^*$. We believe that this is merely a sufficient condition arising from our proof technique, and our numerical evidence shows that the algorithm succeeds for more general $\bS$ and $L^*$.

\textbf{Time complexity.} Each iteration of exact projected-gradient approach needs a full EVD (similar to IHT-type algorithms), which requires cubic running time (computing the gradient needs only needs $\O(pr+r^3)$ operations). Since the total number of iterations is logarithmic, the overall running time scales as $\widetilde{\O}(p^3)$.

The above running time is cubic, and can be problematic for very large $p$. Here, we show that  MAPLE (without imposing the PSD constraint) can successfully reduce the cubic time complexity to \emph{nearly quadratic} in $p$. All we need to do is to provide conditions under which the assumption of RSC/RSS in Theorem~\ref{AppSVD} are satisfied. We achieve this via the following theorem.

\begin{theorem}[RSC/RSS conditions for MAPLE]
\label{RSCRSSapp}  %$n=\O\left(\frac{1}{\delta^{\'2}}\left(\frac{\nu\eta}{1-\rho}\right)^2rp\right)$ for some small constant $\delta^{\'}>0$, with $\rho$ as defined in theorem~\ref{AppSVD}
Let $n=\O\left(pr\right)$. Also, assume the followings for some $C_4,C_3''>0$:
\begin{align}
\label{TrueLapp}
&\|L^*\|_2\leq\frac{1}{1+\sqrt{r^*}}\left(\frac{S_p}{1+C_4\left((1-\epsilon)(\frac{r^*}{r})\right)^{\frac{1}{8}}} - \frac{S_1(C_4((1-\epsilon)(\frac{r^*}{r}))^{\frac{1}{8}})}{1+C_4\left((1-\epsilon)(\frac{r^*}{r})\right)^{\frac{1}{8}}} - \frac{c_2\nu\eta}{1-\rho}\sqrt{\frac{rp}{n}}\right),\\
&\hspace{18mm}S_p\leq S_1\leq \frac{C_3''}{(1-\epsilon)^{\frac{1}{8}}}(\frac{r}{r^*})^{\frac{1}{8}}(S_p-a^{\'}) - \left(1+\sqrt{r^*}\right)\|L^*\|_2- \delta^{\'},
\end{align}
where $0<a^{\'}\leq\left(1+\sqrt{r^*}\right)\|L^*\|_2 + \delta^{\'}$ for some $\delta^{\'}>0$.
Then, the loss function $F(L)$ in~\eqref{opt_probPME} satisfies RSC/RSS conditions with constants $m_{2r+r^*}\geq \frac{1}{(S_1 +\left(1+\sqrt{r}\right)\|L^*\|_2+ \delta')^2}$ and $M_{2r+r^*}\leq \frac{1}{(S_p-a')^2}$ that satisfy the assumptions of Theorem~\ref{AppSVD} in each iteration. 
\end{theorem}

Theorem~\ref{RSCRSSapp} specifies a family of true precision matrices $\Theta^* = \bS + L^*$ that can be provably estimated using our approach with an optimal number of samples. Note that since we do not perform PSD projection within MAPLE, it is possible that some of the eigenvalues of $L^t$ are negative. Next, we show that with high probability, the absolute value of the minimum eigenvalue of $L^t$ is small.

\begin{theorem}
\label{mineigval}
Under the assumptions in Theorem~\ref{RSCRSSapp} on $L^*$, using MAPLE to generate a rank-$r$ matrix $L^t$ for all $t=1,\ldots,T$ guarentees with high probability the minimum eigenvalue of $L^t$ satisfies: $
\la_p(L^t)\geq -a^{\'}$ where $0<a^{\'}\leq\left(1+\sqrt{r^*}\right)\|L^*\|_2+ \frac{c_2\nu\eta}{1-\rho}\sqrt{\frac{rp}{n}}$.
\end{theorem}   

\textbf{Time complexity.} Each iteration of MAPLE needs a tail approximate projection on the set of rank $r$ matrices. According to~\cite{musco2015randomized}, these operations  takes $k' = \O\left(\frac{p^2r\log p}{\sqrt{\epsilon}}\right)$ for approximate ratio $\epsilon$ (computing gradient needs only needs $\O(pr+r^3)$). Since the total number of iterations is once again logarithmic, the overall running time scales as $\widetilde{O}(p^2 r)$. 

\textbf{Sample complexity.} Using the upper bounds in~\eqref{BGRExc} and Theorems~\ref{RSCRSSex} and ~\ref{RSCRSSapp}, the sample complexity of MAPLE scales as $n = \O(pr)$ to achieve a given level of estimation error. From a statistical perspective, this matches, up to constant factors, the number of degrees of freedom of a $p \times p$ matrix with rank $r$.

%% file: res.tex
\section{Experimental results}
\label{sec:expe}
We provide a range of numerical experiments supporting our proposed algorithm and comparing with existing approaches. For NLARM and logistic PCA frameworks, we compare our algorithms with factorized gradient descent~\cite{bhojanapalli2016dropping} as well as projected gradient descent (i.e., the SVP algorithm of~\cite{jain2014iterative}). In our results below, FGD denotes factorized gradient descent algorithm, and SVD refers to SVP-type algorithms where exact SVDs are used for the projection step.\footnote{We have also used the more well-known (but gap-dependent) Lanczos approximation method for the projection step, and have obtained the same performance as full SVD.} For the PME application, our comparisons is with the regularized maximum likelihood approach of~\cite{chandrasekaran2012latent}, which we compare with CVX~\cite{cvx}, and a modification of the ADMM-type method proposed by~\cite{ma2013alternating} (SVD denotes the exact projected-gradient approach). We manually tuned step-sizes and regularization parameters in the different algorithms to achieve the best possible performance. %First we provide the experiments for our first instantiation, NLARM.

\subsection{Nonlinear Affine Rank Minimization}
\begin{figure*}
%\hskip -5mm
\centering
\begin{tabular}{ccc}
\begin{tikzpicture}[scale=0.7]
\begin{axis}[
                width=5cm,
                height=4.5cm,
                scale only axis,
                xmin=0, xmax=60,
                xlabel = {Time (sec)},
                xmajorgrids,
                ymin=-6, ymax=0,
                ylabel={Relative Error},
                ymajorgrids,
                %title={myplot},
%                axis lines*=left,
                line width=1.0pt,
                mark size=1.5pt,
                legend style={at={(0.54,.67)},anchor=south west,draw=black,fill=white,align=left}
                ]
\addplot  [color=blue,
               dashed, 
               very thick,
               mark=o,
               mark options={solid,scale=.5},
               ]
               table [x index=0,y index=1]{aT.txt};
\addlegendentry{FGD}
\addplot [color=red,
               dotted, 
               very thick,
               mark=square,
               mark options={solid,scale=.5},
               ]
               table [x index=2,y index=3]{aT.txt};
\addlegendentry{SVD}
%\addplot [color=green,
%               solid, 
%               very thick,
%               mark=star,
%               mark options={solid,scale=.5},
%               ]
%               table [x index=4,y index=5]{aT.txt};
%\addlegendentry{SVDs}            
\addplot [color=orange,
               solid, 
               very thick,
               mark= diamond,
               mark options={solid,scale=.5},
               ]
               table [x index=6,y index=7]{aT.txt};
\addlegendentry{MAPLE}               
\end{axis}
\end{tikzpicture}&
%%%%%%%%%%%%%%%%%%%%%%%%%%%%%%%%%%%%%
\begin{tikzpicture}[scale=0.7]
\begin{axis}[
                width=5cm,
                height=4.5cm,
                scale only axis,
                xmin=5, xmax=40,
                xlabel = {Projected Rank, $r$},
                xmajorgrids,
                ymin=0, ymax=0.9,
                ylabel={Relative Error},
                ymajorgrids,
                %title={myplot},
%                axis lines*=left,
                line width=1.0pt,
                mark size=1.5pt,
                legend style={at={(0.54,.67)},anchor=south west,draw=black,fill=white,align=left}
                ]
\addplot  [color=blue,
               dashed, 
               very thick,
               mark=o,
               mark options={solid,scale=1},
               ]
               table [x index=0,y index=1]{bT.txt};
\addlegendentry{FGD}
\addplot [color=red,
               dotted, 
               very thick,
               mark=square,
               mark options={solid,scale=1},
               ]
               table [x index=0,y index=2]{bT.txt};
\addlegendentry{SVD}
%\addplot [color=green,
%               solid, 
%               very thick,
%               mark=star,
%               mark options={solid,scale=.5},
%               ]
%               table [x index=0,y index=3]{bT.txt};
%\addlegendentry{SVDs}            
\addplot [color=orange,
               solid, 
               very thick,
               mark= diamond,
               mark options={solid,scale=1},
               ]
               table [x index=0,y index=4]{bT.txt};
\addlegendentry{MAPLE}               
\end{axis}
\end{tikzpicture}&
%%%%%%%%%%%%%%%%%%%%%%%%%%%%%%%%%%%%%
\begin{tikzpicture}[scale=0.7]
\begin{axis}[
                width=5cm,
                height=4.5cm,
                scale only axis,
                xmin=1, xmax=10,
                xlabel = {$\kappa$ ($2^{\kappa}$ Condition Number)},
                xmajorgrids,
                ymin=0, ymax=1,
                ylabel={Probability of Success},
                ymajorgrids,
                %title={myplot},
%                axis lines*=left,
                line width=1.0pt,
                mark size=1.5pt,
                legend style={at={(0.41,.37)},anchor=south west,draw=black,fill=white,align=left}
                ]
\addplot  [color=blue,
               solid, 
               very thick,
               mark=o,
               mark options={solid,scale=1},
               ]
               table [x index=0,y index=1]{cT.txt};
\addlegendentry{FGD (c=5)}
\addplot  [color=blue,
               dashed, 
               very thick,
               mark=square,
               mark options={solid,scale=1},
               ]
               table [x index=0,y index=2]{cT.txt};
\addlegendentry{FGD (c=8)}
\addplot  [color=blue,
               loosely dashed, 
               very thick,
               mark=triangle,
               mark options={solid,scale=1},
               ]
               table [x index=0,y index=3]{cT.txt};
\addlegendentry{FGD (c=11)}
\addplot [color=red,
               dotted, 
               very thick,
               mark=square,
               mark options={solid,scale=1},
               ]
               table [x index=0,y index=4]{cT.txt};
\addlegendentry{SVD}
%\addplot [color=green,
%               solid, 
%               very thick,
%               mark=star,
%               mark options={solid,scale=.5},
%               ]
%               table [x index=0,y index=5]{cT.txt};
%\addlegendentry{SVDs}            
\addplot [color=orange,
               solid, 
               very thick,
               mark=diamond,
               mark options={solid,scale=1},
               ]
               table [x index=0,y index=6]{cT.txt};
\addlegendentry{MAPLE}               
\end{axis}
\end{tikzpicture}\\
%%%%%%%%%%%%%%%%%%%%%%%%%%%%%%%%%%%%%
\begin{tikzpicture}[scale=0.7]
\begin{axis}[
                width=5cm,
                height=4.5cm,
                scale only axis,
                xmin=0, xmax=60,
                xlabel = {Time (sec)},
                xmajorgrids,
                ymin=-6, ymax=0,
                ylabel={Relative Error},
                ymajorgrids,
                %title={myplot},
%                axis lines*=left,
                line width=1.0pt,
                mark size=1.5pt,
                legend style={at={(0.54,.67)},anchor=south west,draw=black,fill=white,align=left}
                ]
\addplot  [color=blue,
               dashed, 
               very thick,
               mark=o,
               mark options={solid,scale=.5},
               ]
               table [x index=0,y index=1]{aB.txt};
\addlegendentry{FGD}
\addplot [color=red,
               dotted, 
               very thick,
               mark=square,
               mark options={solid,scale=.5},
               ]
               table [x index=2,y index=3]{aB.txt};
\addlegendentry{SVD}
%\addplot [color=green,
%               solid, 
%               very thick,
%               mark=star,
%               mark options={solid,scale=.5},
%               ]
%               table [x index=4,y index=5]{aB.txt};
%\addlegendentry{SVDs}            
\addplot [color=orange,
               solid, 
               very thick,
               mark= diamond,
               mark options={solid,scale=.5},
               ]
               table [x index=6,y index=7]{aB.txt};
\addlegendentry{MAPLE}               
\end{axis}
\end{tikzpicture}&
%%%%%%%%%%%%%%%%%%%%%%%%%%%%%%%%%%%%%
\begin{tikzpicture}[scale=0.7]
\begin{axis}[
                width=5cm,
                height=4.5cm,
                scale only axis,
                xmin=5, xmax=40,
                xlabel = {Projected Rank, $r$},
                xmajorgrids,
                ymin=0, ymax=6,
                ylabel={Time (sec)},
                ymajorgrids,
                %title={myplot},
%                axis lines*=left,
                line width=1.0pt,
                mark size=1.5pt,
                legend style={at={(0.0,.67)},anchor=south west,draw=black,fill=white,align=left}
                ]
\addplot  [color=blue,
               dashed, 
               very thick,
               mark=o,
               mark options={solid,scale=1},
               ]
               table [x index=0,y index=1]{bB.txt};
\addlegendentry{FGD}
\addplot [color=red,
               dotted, 
               very thick,
               mark=square,
               mark options={solid,scale=1},
               ]
               table [x index=0,y index=2]{bB.txt};
\addlegendentry{SVD}
%\addplot [color=green,
%               solid, 
%               very thick,
%               mark=star,
%               mark options={solid,scale=.5},
%               ]
%               table [x index=0,y index=3]{bB.txt};
%\addlegendentry{SVDs}            
\addplot [color=orange,
               solid, 
               very thick,
               mark= diamond,
               mark options={solid,scale=1},
               ]
               table [x index=0,y index=4]{bB.txt};
\addlegendentry{MAPLE}               
\end{axis}
\end{tikzpicture}&
%%%%%%%%%%%%%%%%%%%%%%%%%%%%%%%%%%%%%
\begin{tikzpicture}[scale=0.7]
\begin{axis}[
                width=5cm,
                height=4.5cm,
                scale only axis,
                xmin=0.01, xmax=0.1,
                xlabel = {Noise Level ($\sigma$)},
                xmajorgrids,
                ymin=0, ymax=3,
                ylabel={Relative Error},
                ymajorgrids,
                %title={myplot},
%                axis lines*=left,
                line width=1.0pt,
                mark size=1.5pt,
                legend style={at={(0,.47)},anchor=south west,draw=black,fill=white,align=left}
                ]
\addplot  [color=blue,
               dashed, 
               very thick,
               mark=o,
               mark options={solid,scale=1},
               ]
               table [x index=0,y index=1]{cB.txt};
\addlegendentry{FGD}
\addplot [color=red,
               dotted, 
               very thick,
               mark=square,
               mark options={solid,scale=1},
               ]
               table [x index=0,y index=2]{cB.txt};
\addlegendentry{SVD}
%\addplot [color=green,
%               solid, 
%               very thick,
%               mark=star,
%               mark options={solid,scale=.5},
%               ]
%               table [x index=0,y index=3]{cB.txt};
%\addlegendentry{SVDs}            
\addplot [color=orange,
               solid , 
               very thick,
               mark= o,
               mark options={solid,scale=1},
               ]
               table [x index=0,y index=4]{cB.txt};
\addlegendentry{MAPLE ($r=10$)}   
\addplot [color=orange,
               dashed, 
               very thick,
               mark= square,
               mark options={solid,scale=1},
               ]
               table [x index=0,y index=5]{cB.txt};
\addlegendentry{MAPLE ($r=25$)}
\addplot [color=orange,
               loosely dashed, 
               very thick,
               mark= triangle,
               mark options={solid,scale=1},
               ]
               table [x index=0,y index=6]{cB.txt};
\addlegendentry{MAPLE ($r=40$)}            
\end{axis}
\end{tikzpicture}\\
$(a)$ & $(b)$ & $(c)$ 
\end{tabular}
\caption{ Comparisons of algorithms with $g(x) = 2x+sin(x)$. (a) Average of the relative error in estimating $L^*$. Parameters: $p =1000$, $r^* = r = 50$, and $n= 4pr$. \textbf{Top:} $\kappa(L^*)= 1.1$. \textbf{Bottom:} $\kappa(L^*)= 20$. (b) Parameters: $p=300$, $\kappa(L^*)= 1.1$, $r^*=10$, and $n= 4pr$. \textbf{Top:} Average of the relative error. \textbf{Bottom:} Average running time. (c) \textbf{Top:} Probability of success. Parameters: $p=300$, $r^* = r =10$. (c) \textbf{Bottom:} Average of the relative error with different noise level. Parameters: $p=300$, $\kappa = 2$, and $n=7pr$.}
%and $n = cpr$ where $c = 5,8,11$ for FGD algorithm and $5$ for other algorithms.
%$r=10,25,40$ for MAPLE and $10$ for the other algorithms, $r^*=10$, $
\label{NLARM}
%\vskip -5mm
\end{figure*}
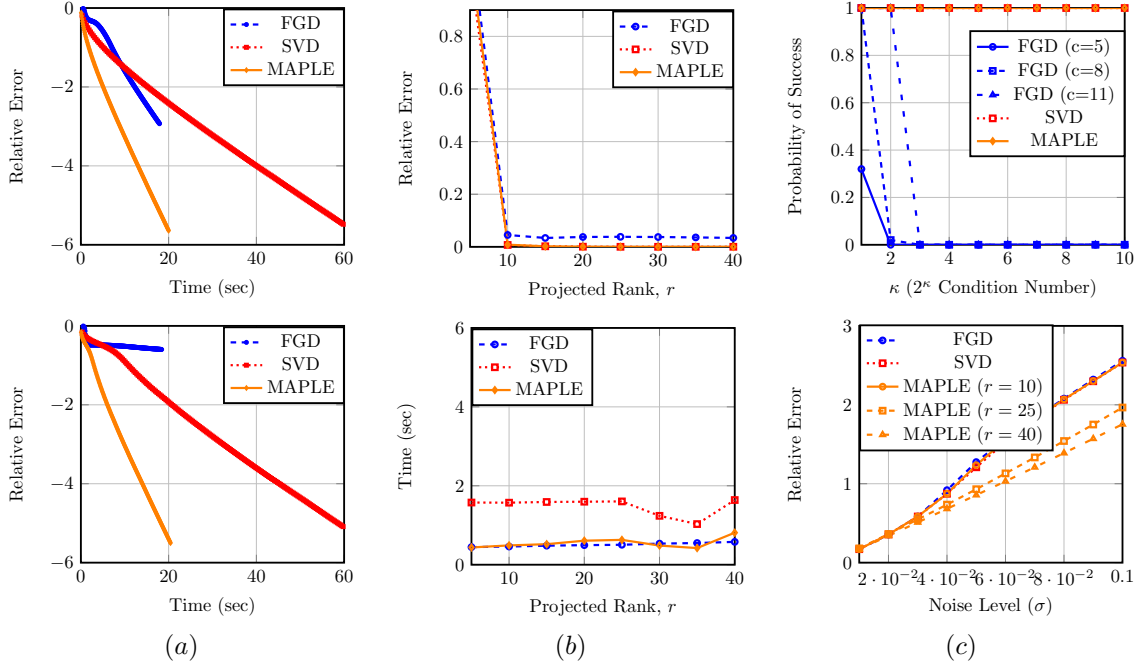

\begin{figure*}
\centering
\begin{tabular}{cccc}
\includegraphics[trim = 48mm 65mm 15mm 65mm, clip, width=0.21\linewidth]{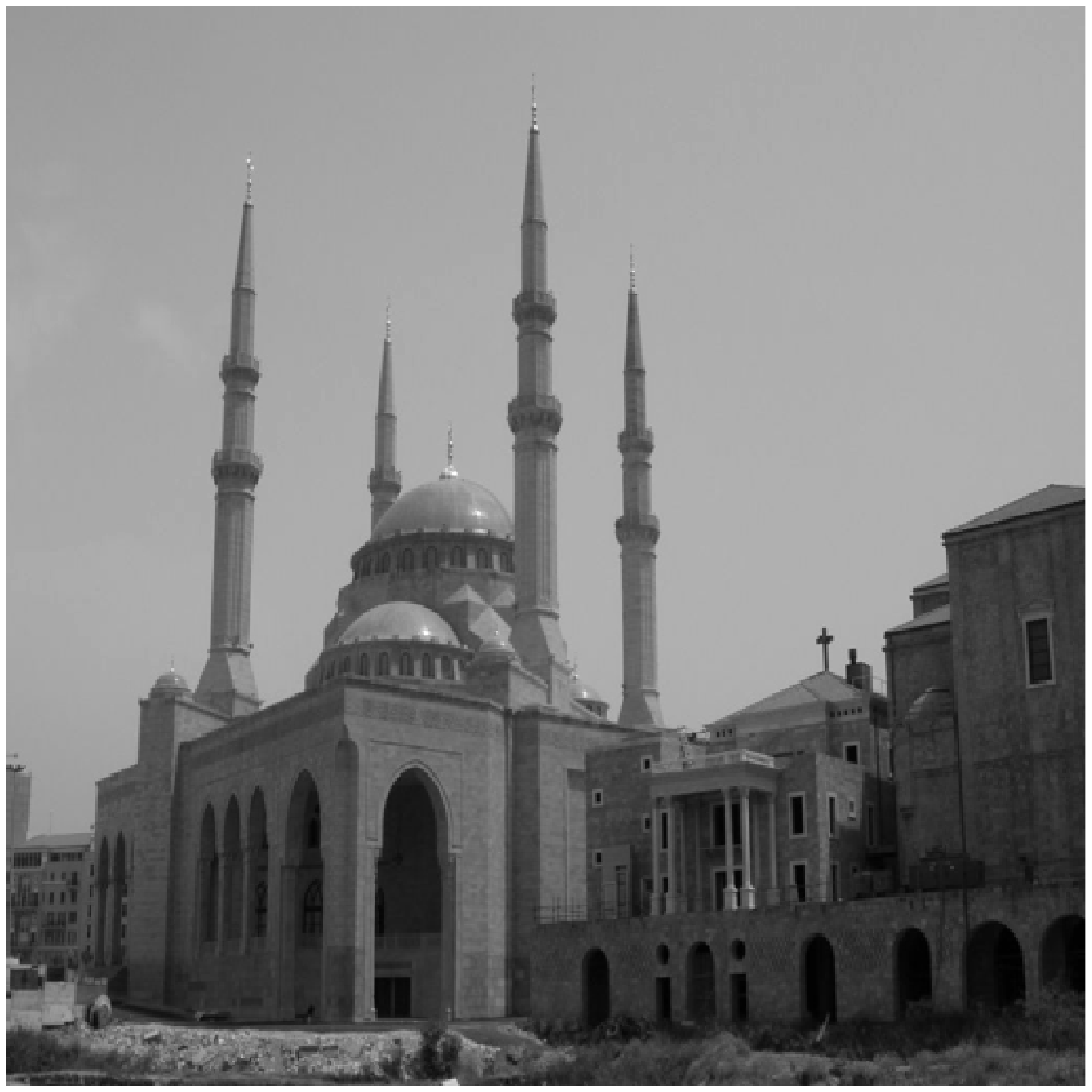} &
\includegraphics[trim = 48mm 65mm 15mm 65mm, clip, width=0.21\linewidth]{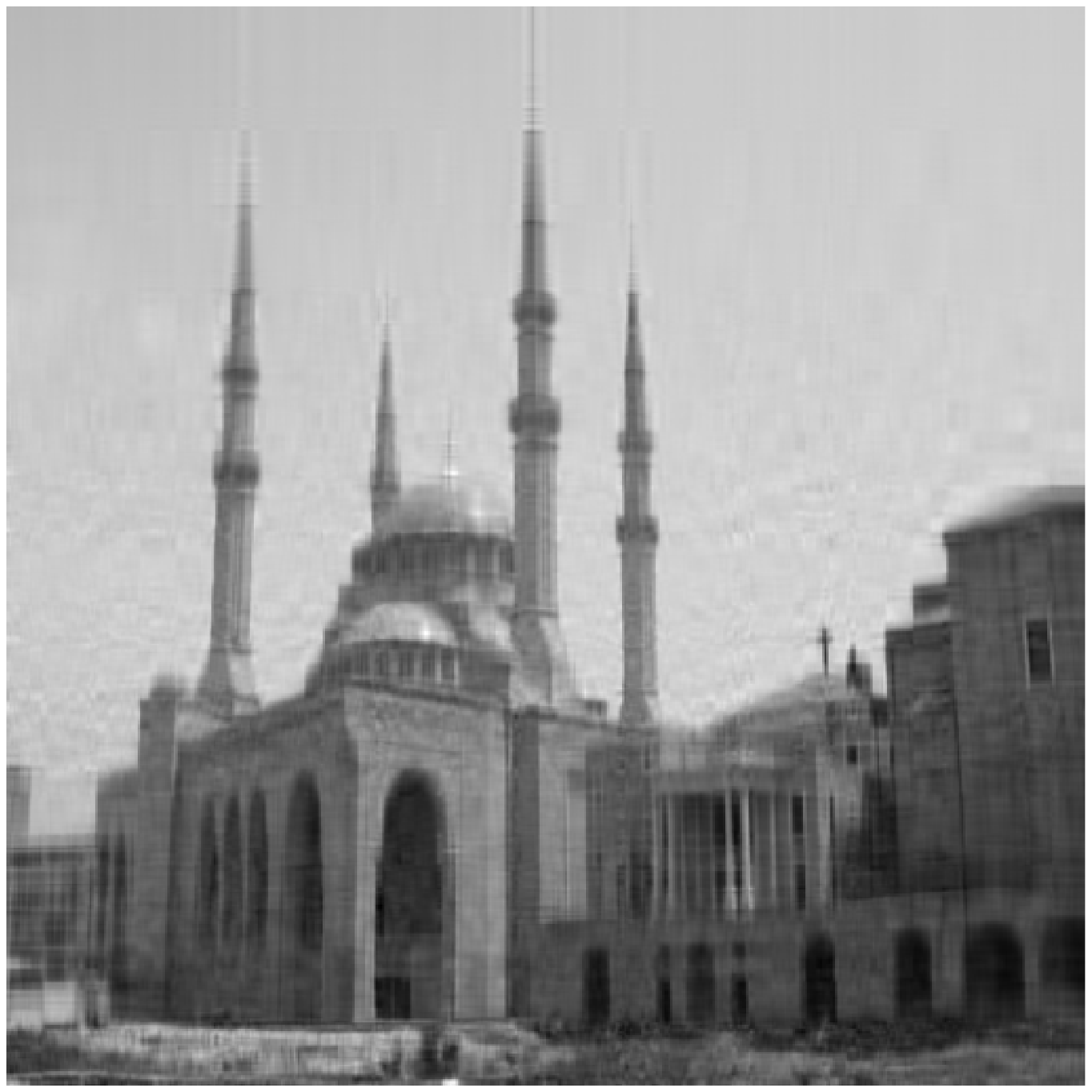}&
\includegraphics[trim = 48mm 65mm 15mm 65mm, clip, width=0.21\linewidth]{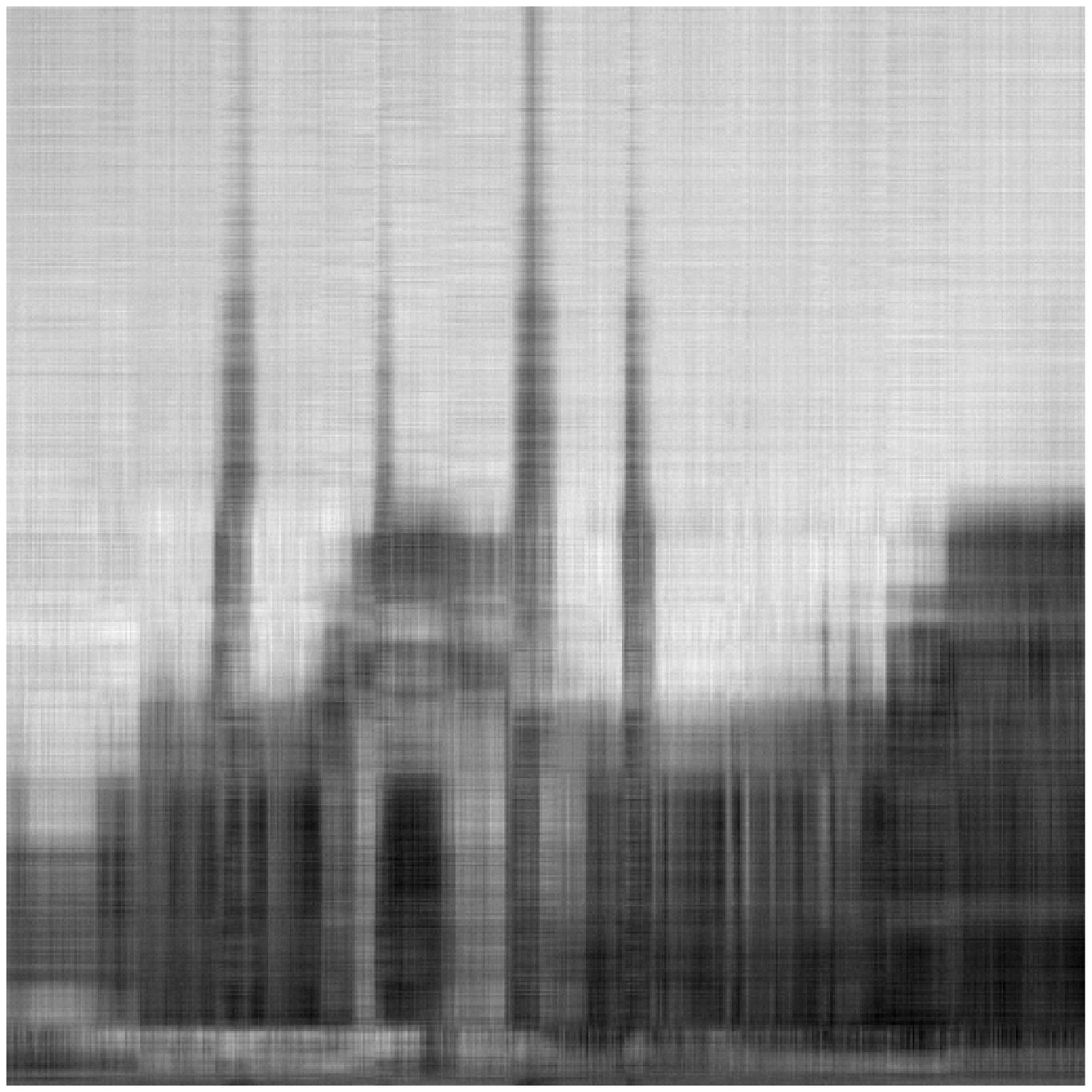}&
\includegraphics[trim = 48mm 65mm 15mm 65mm, clip, width=0.21\linewidth]{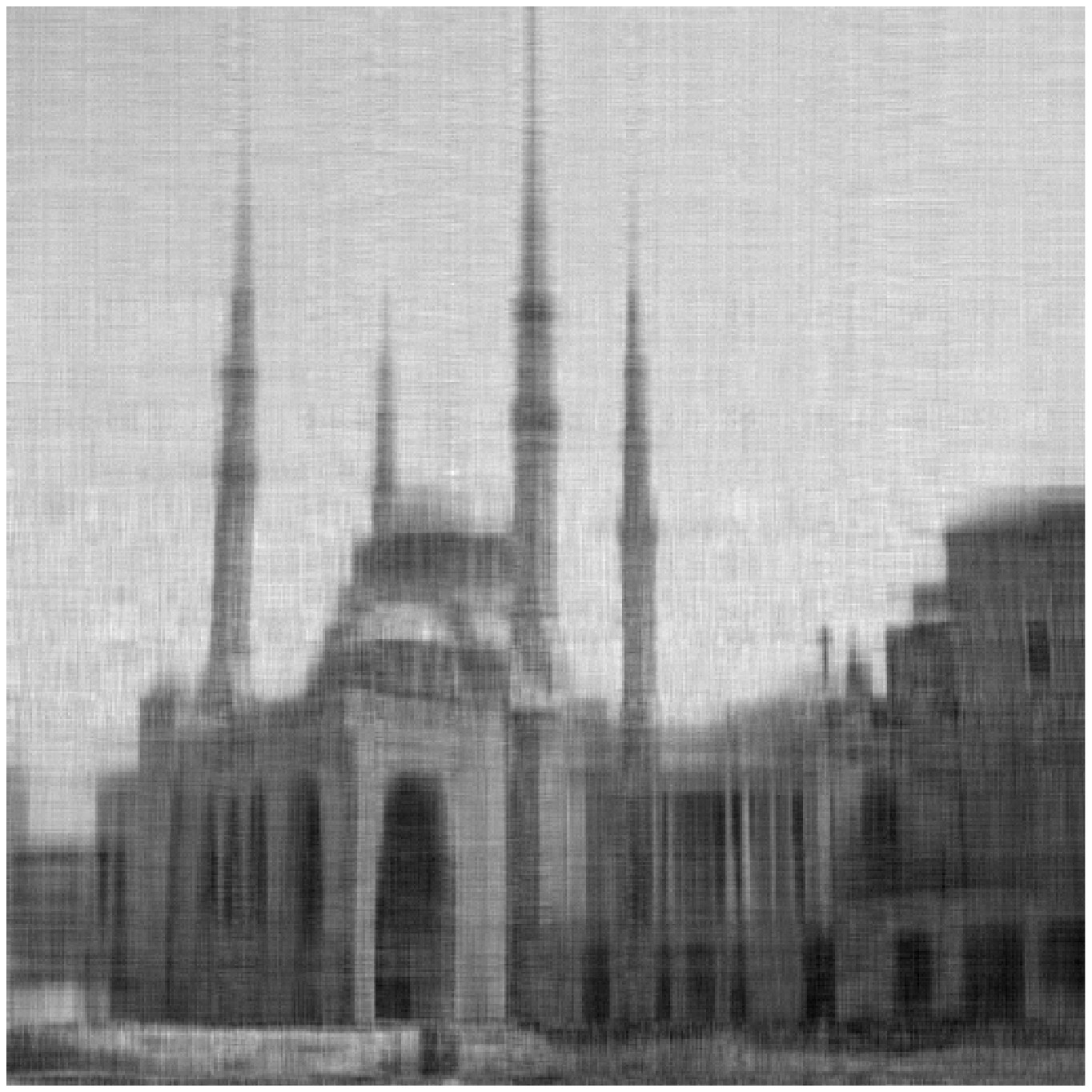}\\
$(a)$ & $(b)$ & $(c)$ & $(d)$\\
\includegraphics[trim = 48mm 65mm 15mm 65mm, clip, width=0.21\linewidth]{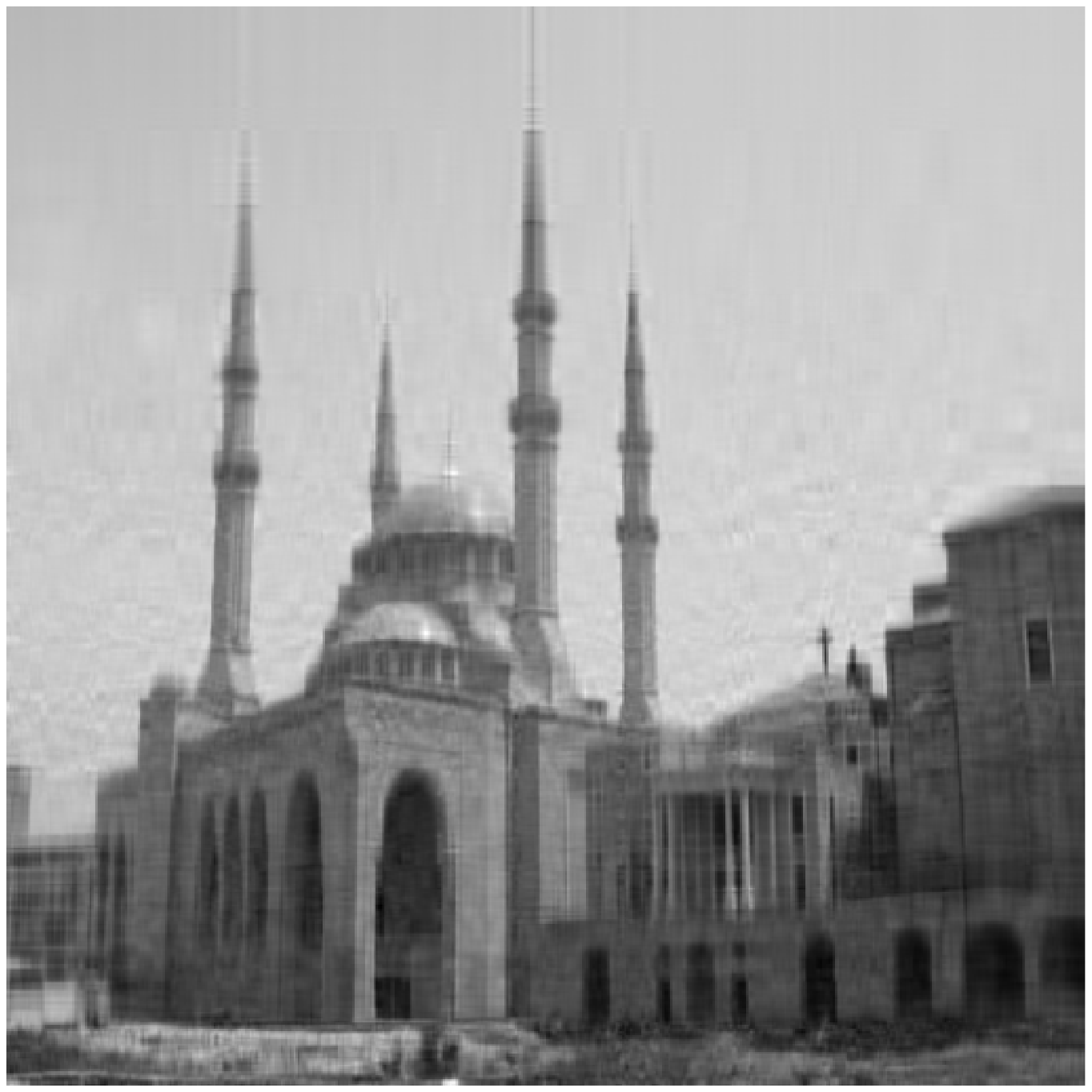} &
\includegraphics[trim = 48mm 65mm 15mm 65mm, clip, width=0.21\linewidth]{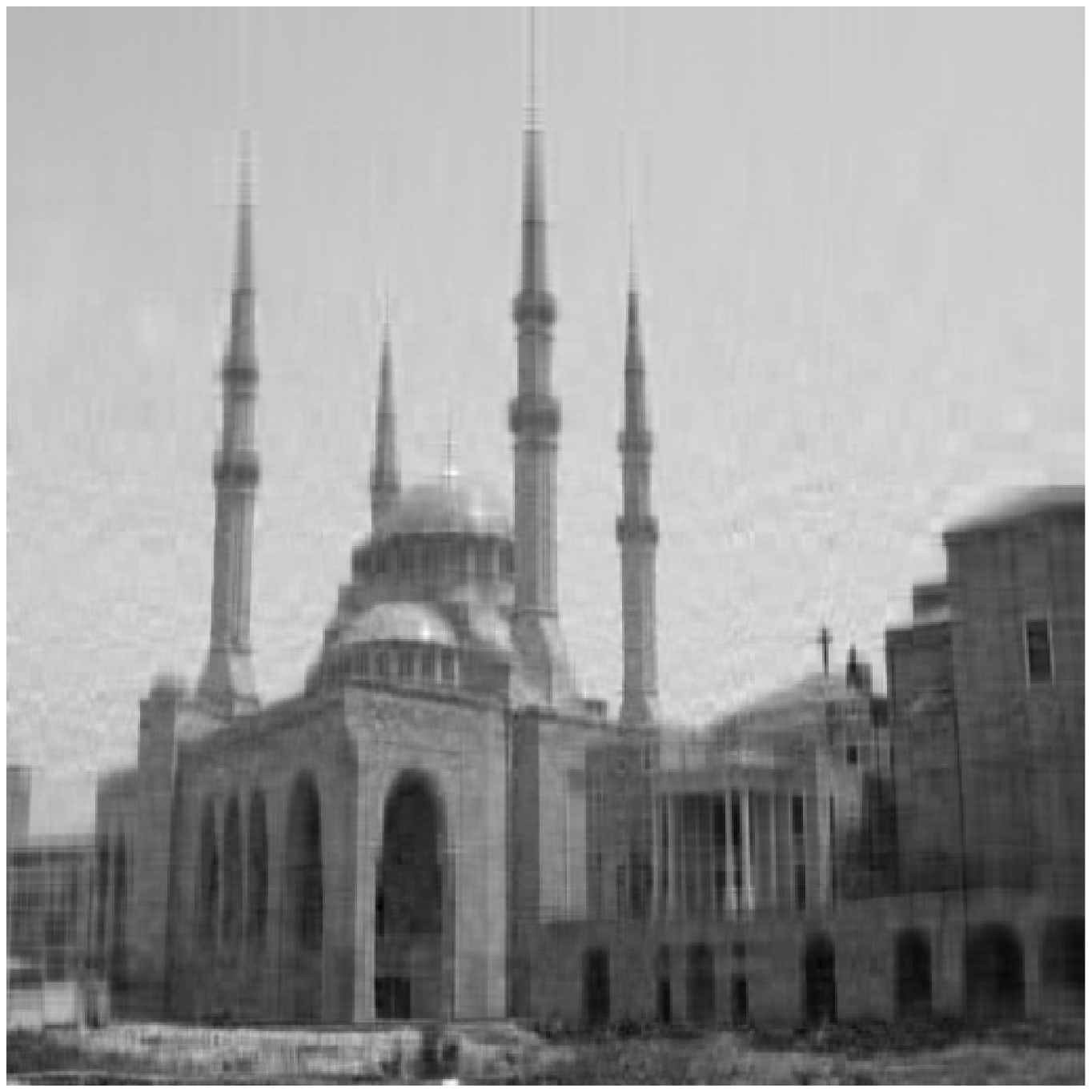}&
\includegraphics[trim = 48mm 65mm 15mm 65mm, clip, width=0.21\linewidth]{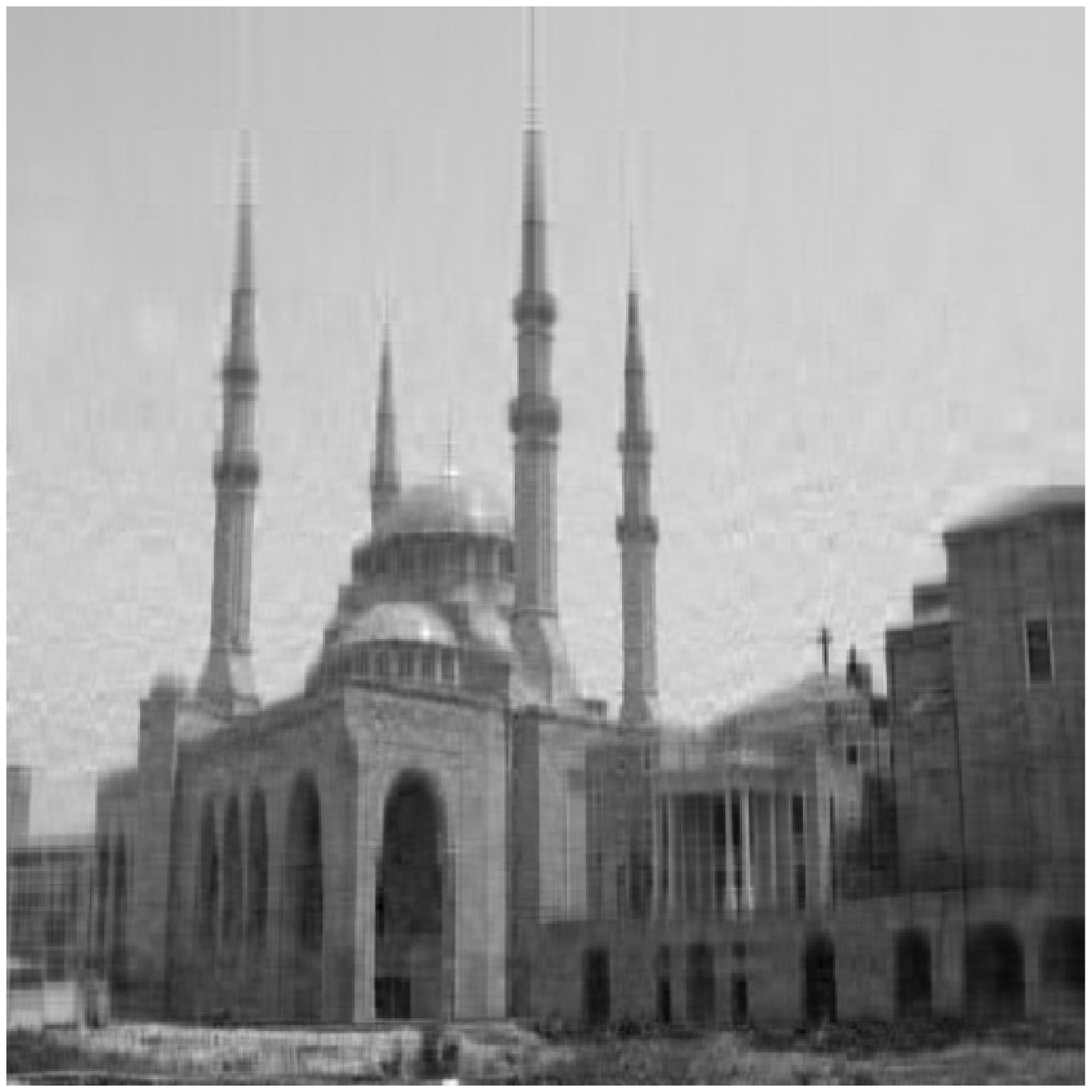}&
\includegraphics[trim = 48mm 65mm 15mm 65mm, clip, width=0.21\linewidth]{Reco_L_BKSVD_2}\\
$(e)$ & $(f)$ & $(g)$ & $(h)$
\end{tabular}
\caption{Comparison of algorithms for real 2D image with $g(x) = \frac{1-e^{-x}}{1+e^{-x}}$. (a) True $512\times 512$ image. (b) Truncated true image with $30$ top singular values. Reconstructed image using (c) FGD, (d) FGD with longer time, (e) SVD, (f) MAPLE ($r = 30$), (g) MAPLE ($r=40$), (h) MAPLE ($r = 50$). }
\label{Real2D}
%\vskip -5mm
\end{figure*}

\begin{table}[t]
\caption{Numerical results for the real data experiment illustrating in Figure~\ref{Real2D}. $T$ denotes the number of iterations.}
\label{TableRealData}
\begin{center}
\begin{small}
\renewcommand{\arraystretch}{1.5}
\begin{tabular}{lccr}
\hline
Algorithm & Relative Error & Running Time & Projected Rank\\
\hline
FGD ($T=300$) & $0.0879$ & $4.9816$ & $30 $ \\
FGD ($T=1000$) & $0.0602$ & $15.9472$ & $30 $ \\
SVD ($T=300$)& $4.4682e-04$ & $19.4700$ & $30$ \\
MAPLE  ($T=300$) & $9.7925e-05$ & $4.2375$ & $30$\\
MAPLE  ($T=300$) & $9.9541e-05$ & $5.5571$ & $40$\\
MAPLE ($T=300$) & $1.3286e-04$ & $7.1306$ & $50$\\
\hline
\end{tabular}
\end{small}
\end{center}
\end{table}

%\textbf{Synthetic data}. 
We report results for all algorithms in Figure~\ref{NLARM}. The link function is set to $g(x) = 2x + \sin(x)$; this function satisfies the derivative conditions discussed above. We  construct the ground truth low-rank matrix $L^*$ with rank $r^*$ by generating a random matrix $U\in\R^{p\times r^*}$ with entries drawn from the standard normal distribution. We ortho-normalize the columns of $U$, and set $L^* = UDU^T$ where $D\in\R^{r^*\times r^*}$ is a diagonal matrix with $D_{11} = \kappa(L^*)$, and $D_{jj} = 1$ for $j\neq1$. After this, we apply a linear operator $\A$ on $L^*$, i.e., $\A(L^*)_i = \langle A_i,L^*\rangle$ where the choice of $A_i$ has been discussed above. Finally, we obtain the measurements $y=g(\A(L^*))$. When reporting noise robustness, we add a Gaussian noise vector $e\in\R^{m}$ to $g(\A(L^*))$. 

In Figure~\ref{NLARM}(a), the running time of the four algorithms are compared. For this experiment, we have chosen $p=1000$, and the rank of the underlying matrix $L^*$ to be $50$. We also set the projected rank as $r = 50$. The number of measurements is set to $n= 4pr$. We consider a well-conditioned matrix $L^*$ with $\kappa(L^*)= 1.1$ for top plot  and $\kappa = 20$ for the bottom one. Then, we measure the relative error in estimating of $L^*$ in Frobenius norm in log scale versus the CPU time takes for $200$ iterations for all of the algorithms. We run the algorithms for $15$ Monte Carlo trials. As we can see, when $\kappa$ is small, FGD has comparable running time with MAPLE (top plot); on the other hand, when we have ill-posed $L^*$, FGD takes much longer to achieve the same relative error.        

Next, we show the performance of the algorithms when the projected rank is changed. The parameters are as $p=300$, $\kappa(L^*)= 1.1$, $r^*=10$, and $n= 4pr$. We set the number of Monte Carlo trials to $50$. In the top plot in Panel (b), we have plotted the relative error as before versus the various $r$ values by averaging over the trials. As we can see, projecting onto the larger space is an effective and practical strategy to achieve small relative error when we do not know the true rank. %However, this is just an empirical observation for FGD~\cite{bhojanapalli2016dropping}. 
Furthermore, the bottom plot of Panel (b) shows the the average running time for either achieving relative error less than $10^{-4}$, or $100$ iterations versus the projected rank. These results suggest that both FGD and MAPLE have the comparable running when we increase the projected rank, while the other SVP algorithms have much longer running time.   

Next, we consider the effect of increasing condition number of the underlying low-rank matrix $L^*$ on the performance of the different algorithms. To do this, we set $p=300$, and $r^* = r=10$. The number of measurements is set to $cpr$ where $c = 5,8,11$ for FGD and $5$ for others. Then we run all the algorithms $50$ times with different condition numbers ranging from $\kappa=1$ (well-posed) to i.e., $\kappa=1024$ (highly ill-posed). We define the probability of success as the number of times that the relative error is less than $0.001$. As illustrated in the top plot of panel (c), all SVP-type algorithms are always able to estimate $L^*$ even for large condition number, i.e., $\kappa = 1024$, whereas FGD fails. In our opinion, this feature is a key benefit of MAPLE over the current fastest existing methods for low-rank estimation (based on factorization approaches).
%But if we increase the number of measurements by different multiplicative coefficients ($c= 8,11$), we see that FGD shows better performance.   

Finally, we consider the noisy scenario in which the observation $y$ is corrupted by different Gaussian noise level. The parameters are set as $p=300$, $r=10,25,40$ for MAPLE and $10$ for the others, $r^*=10$, $n=7pr$, and $\kappa = 2$. The bottom plot in Panel (c) shows the averaged over $50$ trials of the relative error in $L^*$ versus the various standard deviations. From this plot, we see that MAPLE with $r=40$ is most robust, indicating that projection onto the larger subspace is beneficial when noise is present.

%\textbf{Real data}. 
We also run MAPLE on a real 2D $512\times 512$ image, assumed to be an approximately low-rank matrix. The choice of $\A$ is as before, but for the link function, we choose the sigmoid $g(x) = \frac{1-e^{-x}}{1+e^{-x}}$. Figure~\ref{Real2D} visualizes the reconstructed image by different algorithms. In Figure~\ref{Real2D}, (a) is the true image and (b) is the same image truncated to its $r^* = 30$ largest singular values. The result of FGD is shown in (c) and (d) where for (d) we let algorithm run for many more iterations. Reconstruction by SVD is shown in (e). Finally, (f), (g), and (h) illustrate the reconstructed image by using MAPLE with various rank parameters. The numerical reconstruction error is given in Table~\ref{TableRealData}. MAPLE is the fastest method among all methods, even when performing rank-$r$ projection with $r$ larger than $r^*$. 

%%%%%%%%%%%%%%%%%%%%%%%%%%%%%%%%%%%%%%%%%%%%%%%%%%%%
%%%%%%%%%%%%%%%%%%%%%%%%%%%%%%%%%%%%%%%%%%%%%%%%%%%%
\subsection{Logistic PCA}

\begin{figure*}[t]
\centering
\begin{tabular}{cc}
\begin{tikzpicture}[scale=0.8]
\begin{axis}[
                width=5cm,
                height=4.5cm,
                scale only axis,
                xmin=0, xmax=13.5,
                xlabel = {Time (sec)},
                xmajorgrids,
                ymin=13, ymax=13.45,
                ylabel={Logistic loss},
                ymajorgrids,
                %title={myplot},
%                axis lines*=left,
                line width=1.0pt,
                mark size=1.5pt,
                legend style={at={(0.54,.67)},anchor=south west,draw=black,fill=white,align=left}
                ]
\addplot  [color=blue,
               dashed, 
               very thick,
               mark=o,
               mark options={solid,scale=.4},
               ]
               table [x index=0,y index=1]{./RunTIme_well.txt};
\addlegendentry{FGD}
\addplot [color=red,
               dotted, 
               very thick,
               mark=square,
               mark options={solid,scale=.4},
               ]
               table [x index=2,y index=3]{RunTIme_well.txt};
\addlegendentry{SVD}
%\addplot [color=green,
%               solid, 
%               very thick,
%               mark=star,
%               mark options={solid,scale=.5},
%               ]
%               table [x index=4,y index=5]{RunTIme_well.txt};
%\addlegendentry{SVDs}            
\addplot [color=orange,
               solid, 
               very thick,
               mark= diamond,
               mark options={solid,scale=.4},
               ]
               table [x index=6,y index=7]{RunTIme_well.txt};
\addlegendentry{MAPLE}               
\end{axis}
\end{tikzpicture}&
%%%%%%%%%%%%%%%%%%%%%%%%%%%%%%%%%%%%%
\begin{tikzpicture}[scale=0.8]
\begin{axis}[
                width=5cm,
                height=4.5cm,
                scale only axis,
                xmin=5, xmax=30,
                xlabel = {Projected rank, $r$},
                xmajorgrids,
                ymin=9.1, ymax=9.8,
                ylabel={Logistic loss},
                ymajorgrids,
                %title={myplot},
%                axis lines*=left,
                line width=1.0pt,
                mark size=1.5pt,
                legend style={at={(0.54,.67)},anchor=south west,draw=black,fill=white,align=left}
                ]
\addplot  [color=blue,
               dashed, 
               very thick,
               mark=o,
               mark options={solid,scale=1},
               ]
               table [x index=0,y index=1]{MultiRank.txt};
\addlegendentry{FGD}
\addplot [color=red,
               dotted, 
               very thick,
               mark=square,
               mark options={solid,scale=1},
               ]
               table [x index=0,y index=2]{MultiRank.txt};
\addlegendentry{SVD}
%\addplot [color=green,
%               solid, 
%               very thick,
%               mark=star,
%               mark options={solid,scale=.5},
%               ]
%               table [x index=0,y index=3]{MultiRank.txt};
%\addlegendentry{SVDs}            
\addplot [color=orange,
               solid, 
               very thick,
               mark=diamond,
               mark options={solid,scale=1},
               ]
               table [x index=0,y index=4]{MultiRank.txt};
\addlegendentry{MAPLE}               
\end{axis}
\end{tikzpicture}\\
%%%%%%%%%%%%%%%%%%%%%%%%%%%%%%%%%%%%%
\begin{tikzpicture}[scale=0.8]
\begin{axis}[
                width=5cm,
                height=4.5cm,
                scale only axis,
                xmin=0, xmax=13,
                xlabel = {Time (sec)},
                xmajorgrids,
                ymin=11.7, ymax=13.5,
                ylabel={Logistic loss},
                ymajorgrids,
                %title={myplot},
%                axis lines*=left,
                line width=1.0pt,
                mark size=1.5pt,
                legend style={at={(0.54,.67)},anchor=south west,draw=black,fill=white,align=left}
                ]
\addplot  [color=blue,
               dashed, 
               very thick,
               mark=o,
               mark options={solid,scale=.4},
               ]
               table [x index=0,y index=1]{RunTIme_Bad.txt};
\addlegendentry{FGD}
\addplot [color=red,
               dotted, 
               very thick,
               mark=square,
               mark options={solid,scale=.4},
               ]
               table [x index=2,y index=3]{RunTIme_Bad.txt};
\addlegendentry{SVD}
%\addplot [color=green,
%               solid, 
%               very thick,
%               mark=star,
%               mark options={solid,scale=.5},
%               ]
%               table [x index=4,y index=5]{RunTIme_Bad.txt};
%\addlegendentry{SVDs}            
\addplot [color=orange,
               solid, 
               very thick,
               mark= diamond,
               mark options={solid,scale=.4},
               ]
               table [x index=6,y index=7]{RunTIme_Bad.txt};
\addlegendentry{MAPLE}               
\end{axis}
\end{tikzpicture}&
%%%%%%%%%%%%%%%%%%%%%%%%%%%%%%
\begin{tikzpicture}[scale=0.8]
\begin{axis}[
                width=5cm,
                height=4.5cm,
                scale only axis,
                xmin=1.3, xmax=1164,
                xlabel = {Condition Number, $\kappa$ },
                xmajorgrids,
                ymin=5, ymax=10,
                ylabel={Logistic loss},
                ymajorgrids,
                %title={myplot},
%                axis lines*=left,
                line width=1.0pt,
                mark size=1.5pt,
                legend style={at={(0.36,.47)},anchor=south west,draw=black,fill=white,align=left}
                ]
\addplot  [color=blue,
               solid, 
               very thick,
               mark=o,
               mark options={solid,scale=1},
               ]
               table [x index=0,y index=1]{ConNum.txt};
\addlegendentry{FGD (T=50)}
\addplot  [color=blue,
               dashed, 
               very thick,
               mark=square,
               mark options={solid,scale=1},
               ]
               table [x index=0,y index=2]{ConNum.txt};
\addlegendentry{FGD (T=200)}
\addplot  [color=blue,
               loosely dashed, 
               very thick,
               mark=triangle,
               mark options={solid,scale=1},
               ]
               table [x index=0,y index=3]{ConNum.txt};
\addlegendentry{FGD (T=400)}
\addplot [color=red,
               dotted, 
               very thick,
               mark=square,
               mark options={solid,scale=1},
               ]
               table [x index=0,y index=4]{ConNum.txt};
\addlegendentry{SVD}
%\addplot [color=green,
%               solid, 
%               very thick,
%               mark=star,
%               mark options={solid,scale=1},
%               ]
%               table [x index=0,y index=5]{ConNum.txt};
%\addlegendentry{SVDs}            
\addplot [color=orange,
               solid, 
               very thick,
               mark=diamond,
               mark options={solid,scale=1},
               ]
               table [x index=0,y index=6]{ConNum.txt};
\addlegendentry{MAPLE}               
\end{axis}
\end{tikzpicture}\\
$(a)$ & $(b)$
\end{tabular}
\caption{ Comparisons of the algorithms for the average of the logarithm of the logistic loss. (a) Parameters: $p =1000$, $r^* = r = 5$. \textbf{Top:} $\kappa(L^*)= 1.1699$. \textbf{Bottom:} $\kappa(L^*)=  21.4712$. (b) Parameter: $p=200$. \textbf{Top:} Effect of extending the projected space. \textbf{Bottom:} Effect of increasing the condition number for $r^*=r =5$. $T$ denotes the number of iterations. }
\label{LogisticOCAResults}
%\vskip -5mm
\end{figure*}
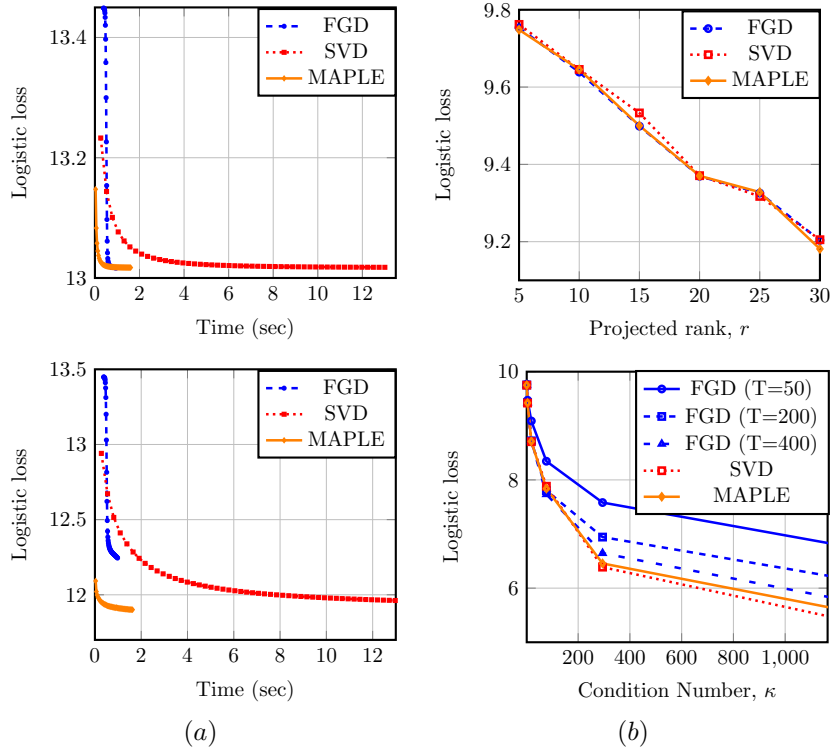

In this section, we provide some representative experimental results for our second application, logistic PCA. We report results for all algorithms in Figure~\ref{LogisticOCAResults}. We  construct the ground truth low-rank matrix $L^*$ with rank $r^*$ similar to NLARM case. 

In panel (a), the running time of all algorithms are compared. For this experiment, we have chosen $p=1000$, and the rank of the underlying matrix $L^*$ to be $5$. We also set the projected rank as $r = 5$. We consider a well-conditioned matrix $L^*$ with $\kappa(L^*)= 1.1531$ for top plot  and $\kappa = 21.4180$ for the bottom one. Then we measure the evolution of the logistic loss defined in~\eqref{logisticPCA} without any regularizer versus the CPU time takes for $50$ iterations for all of the algorithms. We run the algorithms for $20$ Monte Carlo trials, and illustrate the average result. As we can see, when $\kappa$ is small, FGD has comparable running time with MAPLE (top plot); on the other hand, when we have ill-posed $L^*$, FGD takes longer to achieve the same performance.

In panel (b), top plot, we consider the effect of increasing the dimension of the projected space. In this experiment, we set $p=200$, consider the well-posed case where $\kappa(L^*) = 1.4064$, and use $20$ Monte Carlo trials. As we can see all the algorithm show the same trend which verifies that projecting onto the larger space is an effective and practical strategy to achieve small relative error when we do not know the true rank (This is expected according to the theory of MAPLE, while it is not theoretically justified by factorized method).

Finally, the bottom plot in panel (b) shows the effect of increasing condition number of $L^*$. In this experiment, $p=200$, $r^*=r = 5$, and the number of trials equals to $20$. We first let all algorithms run for $50$ iterations, and also consider FGD for more number of iterations, $T=200$ and $T=400$. As it is illustrated, both MAPLE and SVD algorithms are more robust to the large condition number than FGD with $50$ number of iterations. But if we let FGD run longer, it shows the same performance as SVPs which again verifies the dependency of the running time of factorized method to the condition number.

%%%%%%%%%%%%%%%%%%%%%%%%%%%%%%%%%%%%%%%%%%%%%%%%%%%%%%%%%%%%%%%%%%%%%%%%%%%%%%%%%%%%%%%%%%%%%%%%%%%%%%%%

\subsection{Precision Matrix Estimation (PME)}
%We now represents some simulation results for our last application. 
We start first with synthetic datasets.
%\textbf{Synthetic data.} 
We use a diagonal matrix with positive values for the (known) sparse part, $\bS$. For a given number of observed variables $p$, we set $r= 5\%$ as the number of latent variables. We then follow the method proposed in~\cite{ma2013alternating} for generating the sparse and low-rank components $\bS$ and $L^*$. For simplicity, we impose the sparse component to be PSD by forcing it to be positive diagonal matrix. All reported results on synthetic data are the average of 5 independent Monte-Carlo trials.  
%Then, we randomly create a sparse matrix $K\in\R^{(p+r)\times (p+r)}$ with sparsity equal to $2.5\%$ {where the nonzero entries are randomly distributed with standard normal distribution.}  We calculate the matrix $K^TK$ as the true covariance matrix, $\Sigma$. We then choose the sub-matrix $\Theta_{1:p,1:p} = \Sigma^{-1}_{1:p,1:p}$, maintain the diagonal entries, and set the other entries to zero. We then perturb the diagonal entries with some uniform random numbers and make sure that the diagonal entries are positive. This process generates the known sparse part, $S^*$.} 
%For the known sparse part, $S^*$ we randomly generate a diagonal matrix with positive entries which are selected randomly according to the normal distribution. For the low-rank part, we first generate a matrix $K\in\R^{p\times r}$ with entries randomly selected from normal distribution. Then we choose $L^*=K^TK$.
%choose $L^* = \Theta_{(1: p, p + 1: p + r)}\Theta_{(p + 1: p +r, p + 1: p + r)}^{-1}\Theta_{(p + 1: p + r, 1: p)}$. 
Our observations comprise $n$ samples, $x_1,x_2,\ldots,x_n \overset{i.i.d}{\thicksim}\mathcal{N}(0,(\bS+L^*)^{-1})$. In our experiments, we used a full SVD as projection step for exact projected-gradient procedure, FGD, ADMM method and nuclear norm minimization. We used CVX to solve nuclear norm minimization; alternatively, one can use other convex approaches. % which might be faster than convex. 

Panels (a) and (b) in Figure~\ref{allfigsPME} illustrate the comparison of algorithms for PME in terms of the relative error of the estimated $L$ in Frobenius norm versus the ``oversampling" ratio $n/p$. In this experiment, we fixed $p=100$ in (a) and $p=1000$ in (b) and vary $n$. In addition, for both of these results, condition number is given by $\kappa(L^*) = 2.4349$ and  $\kappa(L^*) = 2.9666$, respectively. We observe that MAPLE, FGD, and exact projected gradient descent are able to estimate the low-rank matrix even for the regime where $n$ is very small, whereas both ADMM and CVX does not produce very meaningful results.
\begin{figure}
%\hskip -5.5mm
\centering
\begin{tabular}{ccc}
\begin{tikzpicture}[scale=.7]
\begin{axis}[
                width=5cm,
                height=4.5cm,
                scale only axis,
                xmin=25, xmax=400,
                xlabel = {$n/p$},
                xtick={25,50,100,200,400},
                xticklabels={25,50,100,200,400},
                xmajorgrids,
                ymin=0, ymax=3,
                ylabel={ Relative Frobenius Error},
                ymajorgrids,
                %title={myplot},
%                axis lines*=left,
                line width=1.0pt,
                mark size=1.5pt,
                legend style={ at={(0.54,.48)},anchor=south west,draw=black,fill=white,align=left}
%                legend style={font=\fontsize{.02}{2}\selectfont, at={(0.62,.68)},anchor=south   
%                west,draw=black,fill=white,align=left}
                ]
\addplot  [color=blue,
               dashed, 
               very thick,
               mark=o,
               mark options={solid,scale=1.2},
               ]
               table [x index=0,y index=1]{ReErrp100.txt};
\addlegendentry{FGD}
\addplot [color=red,
               dotted, 
               very thick,
               mark=square,
               mark options={solid,scale=1.2},
               ]
               table [x index=0,y index=2]{ReErrp100.txt};
\addlegendentry{SVD}
\addplot [color=orange,
               solid, 
               very thick,
               mark=diamond,
               mark options={solid,scale=1.2},
               ]
               table [x index=0,y index=3]{ReErrp100.txt};
\addlegendentry{MAPLE}            
\addplot [color=green,
               loosely dashed , 
               very thick,
               mark= star,
               mark options={solid,scale=1.2},
               ]
               table [x index=0,y index=4]{ReErrp100.txt};
\addlegendentry{ADMM}   
\addplot [color=black,
               dashdotted , 
               very thick,
               mark= triangle,
               mark options={solid,scale=1.2},
               ]
               table [x index=0,y index=4]{ReErrp100.txt};
\addlegendentry{Convex}   
\end{axis}
\end{tikzpicture}&
%%%%%%%%%%%%%%%%%%%%%%%%%%%%%%%%%%%%%%%%%
\begin{tikzpicture}[scale=.7]
\begin{axis}[
                width=5cm,
                height=4.5cm,
                scale only axis,
                 xmin=25, xmax=400,
                xlabel = {$n/p$},
                xtick={25,50,100,200,400},
                xmajorgrids,
                ymin=0, ymax=3,
                ylabel={ Relative Frobenius Error},
                ymajorgrids,
                %title={myplot},
%                axis lines*=left,
                line width=1.0pt,
                mark size=1.5pt,
                legend style={at={(0.54,.58)},anchor=south   
                west,draw=black,fill=white,align=left}
                ]
\addplot  [color=blue,
               dashed, 
                very thick,
               mark=o,
               mark options={solid,scale=1.2},
               ]
               table [x index=0,y index=1]{ReErrp1000.txt};
\addlegendentry{FGD}
\addplot [color=red,
               dotted, 
               very thick,
               mark=square,
               mark options={solid,scale=1.2},
               ]
               table [x index=0,y index=2]{ReErrp1000.txt};
\addlegendentry{SVD}
\addplot [color=orange,
               solid, 
                very thick,
               mark=diamond,
               mark options={solid,scale=1.2},
               ]
               table [x index=0,y index=3]{ReErrp1000.txt};
\addlegendentry{MAPLE}            
\addplot [color=green,
               loosely dashed , 
               very thick,
               mark= star,
               mark options={solid,scale=1.2},
               ]
               table [x index=0,y index=4]{ReErrp1000.txt};
\addlegendentry{ADMM}               
\end{axis}
\end{tikzpicture}&
%%%%%%%%%%%%%%%%%%%%%%%%%%%%%%%%%%%%%%%%
\begin{tikzpicture}[scale=.7]
\begin{axis}[
                width=5cm,
                height=4.5cm,
                scale only axis,
                xmin=0, xmax=20,
                xlabel = {Time (sec)},
                xmajorgrids,
                ylabel={ Negative Log-Likelihood},
                ymajorgrids,
                %title={myplot},
%                axis lines*=left,
                line width=1.0pt,
                mark size=1.5pt,
                legend style={ at={(0.55,.68)},anchor=south   
                west,draw=black,fill=white,align=left}
                ]
\addplot  [color=orange,
               solid, 
                thick,
               mark=o,
               mark options={solid,scale=1.2},
               ]
               table [x index=0,y index=1]{RuntimeRosetta.txt};
\addlegendentry{MAPLE}
\addplot [color=red,
               dashed, 
                thick,
               mark=square,
               mark options={solid,scale=1.2},
               ]
               table [x index=2,y index=3]{RuntimeRosetta.txt};
\addlegendentry{SVD}
%\addplot [color=orange,
%               dashdotted, 
%                thick,
%               mark=star,
%               mark options={solid,scale=1.2},
%               ]
%               table [x index=4,y index=5]{RuntimeRosetta.txt};
%\addlegendentry{SVDs}                       
\end{axis}
\end{tikzpicture}\\
$(a)$ & $(b)$ &$(c)$
\end{tabular}
\caption{Comparison of algorithms both in synthetic and real data. (a) relative error of $L$ in Frobenius norm with $p =100$, and $r=r^* = 5$. (b) relative error of $L$ in Frobenius norm with $p =1000$, and $r=r^* = 50$. (c) NLL versus time in Rosetta data set with $p=1000$.}
\label{allfigsPME}
%\vskip -5mm
\end{figure}
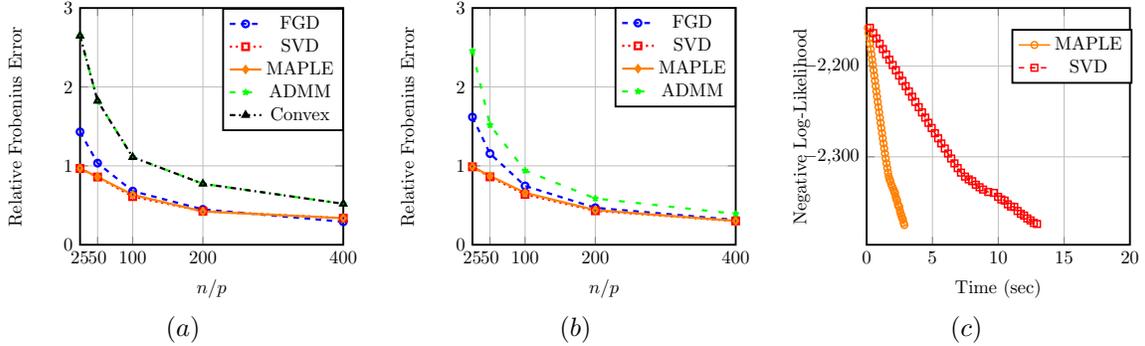
We also report the results of several more experiments on synthetic data. In the first experiment, we set $p=100$, $n=400p$, and $r=r^*=5$. Table~\ref{p100n400ptable} lists several metrics that we use for algorithm comparison. 
%An algorithm terminates if it satisfies one of two conditions: the evaluated objective function in the estimated $L$ in each iteration falls below the true negative log likelihood (NLL) (i.e., $F(L^*)$), or the total number of iterations exceeds $600$. 
From Table~\ref{p100n400ptable}, we see that MAPLE, FGD, and exact procedure produce better estimates of $L$ compared to ADMM and convex method. As anticipated, the total running time of convex approach is much larger than other algorithms. Finally, the estimated objective function for first three  algorithms is very close to the optimal (true) objective function compared to ADMM and CVX. 

We increase the dimension to $p=1000$ and reported the same metrics in Table~\ref{p1000n400ptable} similar to Table~\ref{p100n400ptable}. We did not report convex results as it takes long time to be completed. Again, we get the same conclusions as Table~\ref{p100n400ptable}. Important point here is that in this specific application, FGD has better running time compared to MAPLE for both well-condition and ill-condition problem. Here, we did not report the running time for the ill-posed case; however, we observed that FGD is not affected by condition number of ground-truth. We conjecture that FGD delivers a solution for problem~\eqref{opt_probPME} such that its convergence is independent of the condition number of ground-truth similar to~\cite{bhojanapalli2016global} where authors showed that for linear matrix sensing problem, there is no dependency on the condition number if they use FGD method. Proving of this conjecture can be interesting future direction. Also, Tables~\ref{p100n50ptable} and~\ref{p1000n50ptable} show the same experiment discussed in Tables~\ref{p100n400ptable} and~\ref{p1000n400ptable}, but for small number of samples, $n=50p$.

%{\textbf{Real data} 
Here, we just evaluate our methods through the \emph{Rosetta} gene expression data set~\cite{hughes2000functional}. This data set includes 301 samples with 6316 variables. We run the ADMM algorithm by~\cite{ma2013alternating} with $p=1000$ variables which have highest variances,  and obtained an estimate of the positive definite component $\bS$. Then we used $\bS$ as the input for MAPLE and exact projection procedure. The target rank for all three algorithms is set to be the same as that returned by ADMM. In Figure~\ref{allfigsPME} plot (c), we illustrate the NLL for these algorithms versus wall-clock time (in seconds) over 50 iterations. We observe that all the algorithms demonstrate linear convergence, as predicted in the theory. Among the these algorithms, MAPLE obtains the quickest rate of decrease of the objective function.}
%\vskip -3mm
\begin{table*}
\caption{Comparison of different algorithms for $p=100$ and $n=400p$. NLL stands for negative log-likelihood.}
\label{p100n400ptable}
\vskip 0.1in
\begin{center}
\begin{small}
\begin{sc}
\begin{tabular}{lcccccr}
\hline
%\abovespace
Alg & Estimated NLL & True NLL  & Relative error & Total time \\
\hline 
FGD & $-9.486278e+01 $ & $-9.485018e+01$ & $2.914218e-01$ & $2.150596e-02$ \\
SVD & $9.485558e+01$  & $-9.485018e+01$ & $3.371867e-01$ & $6.552529e-01$  \\
MAPLE  & $-9.485558e+01$ & $-9.485018e+01$ & $3.371742e-01$ & $3.092728e-01$   \\
ADMM & $-9.708976e+01$ & $-9.485018e+01$ & $5.192783e-01$ & $1.475124e+00 $ \\
Convex & $-9.491779e+01$  & $-9.485018e+01$ & $5.192783e-01$ & $7.482316e+02$ \\
\hline
\end{tabular}
\end{sc}
\end{small}
\end{center}
%\vskip -0.1in
\end{table*}

\begin{table*}[ht]
\caption{Comparison of different algorithms for $p=1000$ and $n=400p$.} %The meaning of the columns are the same as Table~\ref{p100n400ptable}}
\label{p1000n400ptable}
%\vskip 0.1in
\begin{center}
\begin{small}
\begin{sc}
\begin{tabular}{lcccccr}
\hline
%\abovespace\belowspace
Alg & Estimated NLL & True NLL  & Relative error & Total time \\
\hline
%\abovespace
FGD & $-2.638684e+03$  & $-2.638559e+03$ & $3.144617e-01$ & $1.301985e+01$ \\
SVD & $-2.638674e+03 $ &$-2.638559e+03$ & $3.019913e-01$ & $1.584453e+02$  \\
MAPLE & $-2.638675e+03$ & $-2.638559e+03$ & $3.020130e-01$ & $2.565310e+01$  \\
ADMM & $-2.638920e+03$ & $-2.638559e+03$ & $3.921407e-01$ & $3.375073e+02$ \\
\hline
\end{tabular}
\end{sc}
\end{small}
\end{center}
%\vskip -0.1in
\end{table*}

\begin{table*}[!t]
\caption{Comparison of different algorithms for $p=100$ and $n=50p$. NLL stands for negative log-likelihood.}
\label{p100n50ptable}
%\vskip 0.1in
\begin{center}
\begin{small}
\begin{sc}
\begin{tabular}{lcccccr}
\hline
%\abovespace\belowspace
Alg & Estimated NLL & True NLL  & Relative error & Total time \\
\hline
%\abovespace
FGD & $-9.483037e+01$  & $-9.470944e+01$ & $1.034812e+00$ & $2.294928e-02$ \\
SVD  & $-9.477855e+01$ & $-9.470944e+01$ & $8.586494e-01$ & $1.026811e+00$  \\
MAPLE  & $-9.478611e+01$ & $-9.470944e+01$ & $8.606593e-01$ & $4.854349e-01$  \\
ADMM & $-9.356307e+01$ & $-9.470944e+01$ & $1.823421e+00$ & $3.001534e+00$  \\
Convex & $-9.528296e+01$  & $-9.470944e+01$ & $1.864212e+00$ & $7.046295e+02$    \\
\hline
\end{tabular}
\end{sc}
\end{small}
\end{center}
%\vskip -0.1in
\end{table*}

\begin{table*}[ht]
\caption{Comparison of different algorithms for $p=1000$ and $n=50p$. %The meaning of the columns are the same as Table~\ref{p100n50ptable}.
}
\label{p1000n50ptable}
%\vskip 0.1in
\begin{center}
\begin{small}
\begin{sc}
\begin{tabular}{lcccccr}
\hline
%\abovespace\belowspace
Alg & Estimated NLL & True NLL  & Relative error & Total time \\
\hline
%\abovespace
FGD & $-2.639646e+03$  & $-2.638491e+03$ & $1.155856e+00$ & $1.335701e+01$ \\
SVD & $-2.638804e+03$ & $-2.638491e+03$ & $8.610451e-01$ & $1.567543e+02$  \\
MAPLE & $-2.638878e+03$ & $-2.638491e+03$ & $8.722342e-01$ & $2.606750e+01$  \\
ADMM & $-2.643757e+03$ & $-2.638491e+03$ & $1.517834e+00$ & $4.019458e+02$ \\
\hline
\end{tabular}
\end{sc}
\end{small}
\end{center}
%\vskip -0.1in
\end{table*}

%% file: appen.tex
%\section{Appendix}
\label{append}

\subsection{Proofs}
We provide full proofs of all theorems discussed in this paper. 

 %Finally $\kappa(A) = \frac{\la_1(A)}{\la_p(A)}$ denotes the condition number of the matrix $A$.
%{Below, the expression $C+D$ for two sets $C$ and $D$ refers to the \textit{Minkowski} sum of two sets, defined as $C+D = \{c+d \ | \ c\in C, \ d\in D\}$ for given sets $C$ and $D$.  }
Below, $\M(\U_r)$ denotes the set of vectors associated with $\U_r$, the set of all rank-r matrix subspaces. We show the maximum and minimum eigenvalues of a matrix $A\in\R^{p\times p}$ as $\lambda_{\min}(A), \lambda_{\max}(A)$, respectively. Furthermore $\sigma_i(A)$ denotes the $i^{th}$ largest singular value of matrix $A$.   
We need the following equivalent definitions of restricted strongly convex and restricted strong smoothness conditions.

\begin{definition} \label{defRSCRSS_app}
A function $f$ satisfies the Restricted Strong Convexity (RSC) and Restricted Strong Smoothness (RSS) conditions if one of the following equivalent definitions is satisfied for all $L_1,L_2, L\in\mathbb{R}^{p\times p}$ such that $\rank(L_1)\leq r, \rank(L_2)\leq r,rank(L)\leq r$:
\begin{align}
%\label{rscrss_app}
&\frac{m_r}{2}\|L_2-L_1\|^2_F \leq f(L_2) - f(L_1) - \langle\nabla f(L_1) , L_2-L_1\rangle\leq\frac{M_{r}}{2}\|L_2-L_1\|^2_F, \label{rscrss_app1}\\
&\hspace{5mm}m_{r}\|L_2-L_1\|^2_F \leq\langle\P_U\left(\nabla f(L_2) - \nabla f(L_1)\right), L_2-L_1\rangle \leq M_{r}\|L_2-L_1\|^2_F, \label{rscrss_app2}\\
&\hspace{35mm}m_{r} \leq\|\P_U\nabla^2 f(L)\|_2 \leq M_{r},\label{rscrss_app3} \\
&\hspace{7mm}m_{r}\|L_2-L_1\|_F \leq\|\P_U\left(\nabla f(L_2) - \nabla f(L_1)\right)\|_F \leq M_{r}\|L_2-L_1\|_F,\label{rscrss_app4}
\end{align} 
where $U$ is the span of the union of column spaces of the matrices $L_1$ and $L_2$. Here, $m_r$ and $M_r$ are the RSC and RSS constants, respectively. 
\end{definition}

\subsection{Proof of theorems in section~\ref{prelim} }
Before proving the main theorems, we restate the following hard-thresholding result from lemma $3.18$ in~\cite{li2016Nonconvex}:
\begin{lemma}\label{HT}
For $r>r^*$ and for any matrix $L\in\R^{p\times p}$, we have
\begin{align}
\|H_r(L)-L^*\|_F^2\leq\left(1+\frac{2\sqrt{r^*}}{\sqrt{r-r^*}}\right)\|L-L^*\|_F^2,
\end{align}
where $\rank(L^*)$ = $r^*$, and $H_r(.):\R^{p\times p}\rightarrow\U_{r}$ denotes the singular value thresholding operator, which keeps the largest $r$ singular values and sets the others to zero.
\end{lemma}
For proving theorem~\ref{AppSVD}, we cannot use directly lemma~\ref{HT} since $\Ta$ operator returns an approximation of the top $r$ singular vectors, and using exact projection in the proof of lemma~\ref{HT} is necessary~\cite{li2016Nonconvex}. However, we can modify the proof of lemma~\ref{HT} to make it applicable through the approximate projection approach. Hence, we can prove Lemma~\ref{ApprHTh}:
%\begin{lemma}~\label{ApprHTh}
%For $r>(1+\frac{1}{1-\epsilon})r^*$ and for any matrix $L\in\R^{p\times p}$, we have
%\begin{align}
%\|\Ta(L)-L^*\|_F^2\leq\left(1+\frac{2}{\sqrt{1-\epsilon}}\frac{\sqrt{r^*}}{\sqrt{r-r^*}}\right)\|L-L^*\|_F^2,
%\end{align}
%where $\rank(L^*)$ = $r^*$, $\Ta:\R^{p\times p}\rightarrow\U_{r}$ denotes the approximate tail projection defined in the definition~\ref{taildef} and $\epsilon>0$ is the approximation ratio introduced by the approximate projection.
%\end{lemma}
 
\begin{proof}[Proof of Lemma~\eqref{ApprHTh}]
The proof is similar to the procedure described in~\cite{li2016Nonconvex} with some modification based on the per-vector guarantee property of approximate projection. In this work, the proof in is given first for sparse hard thresholding, and then is generalized to the low-rank case using Von Neumann's trace inequality, i.e., for two matrices $A,B\in\R^{p\times p}$ and corresponding singular values $\sigma_i(A)$ and  $\sigma_i(B)$, respectively, we have: 
\begin{align}\label{VoNeu}
\langle A,B\rangle = \Sigma_{k=1}^{\min\{\rank(A),\rank(B)\}}\sigma_k(A)\sigma_k(B).
\end{align}
First define $\theta = [\sigma_1^2(L),\sigma_2^2(L)\ldots,\sigma_r^2(L)]^T$. Let $\theta^* = [\sigma_1^2(L^*),\sigma_2^2(L^*)\ldots,\sigma_r^2(L^*)]^T$, and $\theta' = \Ta(\theta)$. Also, let $supp(\theta^*)  = \mathcal{I^*}$, $supp(\theta)  = \mathcal{I}$, $supp(\theta')  = \mathcal{I'}$, and $\theta'' =\theta - \theta'$ with support $I''$. It follows that $$\|\theta'-\theta^*\|_2^2 - \|\theta- \theta^*\|_2^2\leq 2\langle\theta'',\theta^*\rangle - \|\theta''\|_2^2.$$ Now define new sets $\mathcal{I^*}\cap\mathcal{I'} = \mathcal{I}^{*1}$ and $\mathcal{I^*}\cap\mathcal{I''} = \mathcal{I}^{*2}$ with restricted vectors to these sets as $\theta_{\mathcal{I}^{*1}} = \theta^{*1}$, $\theta_{\mathcal{I}^{*2}} = \theta^{*2}$, $\theta'_{\mathcal{I}^{*1}} = \theta^{1*}$, and $\theta''_{\mathcal{I}^{*2}} = \theta^{2*}$ such that $|\mathcal{I}^{*2}| = r^{**}$. Hence, $\|\theta^{2*}\|_2 = \beta\theta_{\max}$ where $\beta\in[\sqrt{r^{**}}]$ and $\theta_{\max} = \|\theta^{2*}\|_{\infty}$. By these definitions, we have:
$$ \|\theta'-\theta^*\|_2^2 - \|\theta- \theta^*\|_2^2\leq2\|\theta^{2*}\|_2\|\theta^{*2}\|_2 - \|\theta^{2*}\|_2^2.$$
The proof continues to discuss in three cases as:
\begin{enumerate}
\item if $\|\theta^{2*}\|_2\leq \theta_{\max}$, then $\beta=1$.
\item if $\theta_{\max}\leq\|\theta^{2*}\|_2< \sqrt{r^{**}}\theta_{\max}$, then $\beta = \frac{\|\theta^{2*}\|_2}{\theta_{\max}}$.
\item if $\|\theta^{2*}\|_2\geq\sqrt{r^{**}}\theta_{\max}$, then $\beta = \sqrt{r^{**}}$.
\end{enumerate}
In each case, the ratio of $\frac{\|\theta'-\theta^*\|_2^2 - \|\theta- \theta^*\|_2^2}{\|\theta- \theta^*\|_2^2}$ is upper bounded in terms of $r,r^*, r^{**}$ and by using the inequality $|\theta_{\min}|\geq|\theta_{\max}|$ where $|\theta_{\min}|$ is defined as the smallest entry of $\theta^{1*}$. This inequality holds due to the exact hard thresholding. However, it does not necessary hold when approximate projection is used. To resolve this problem, we note that in our framework the approximate tail projection is implemented via any randomized SVD method which supports the so-called per-vector guarantee. Recall from our discussion in section~\ref{prelim}, the per vector guarantee condition means: 
$$|u_i^TLL^Tu_i - z_iLL^Tz_i|\leq\epsilon\sigma_{r+1}^2\leq\epsilon\sigma_{i}^2, \ \ \ \ \ i\in[r].$$ In our implementation, we use randomized block Krylov method (BK-SVD) which supports this condition. In our notations, this condition implies, $|\theta_{\min} - \widehat{\theta}_{\min}|\leq\epsilon\theta_{\min}$ where $\widehat{\theta} = [\widehat{\sigma}_1^2(L),\widehat{\sigma}_2^2(L)\ldots,\widehat{\sigma}_r^2(L)]^T$. By combing with $|\theta_{\min}|\geq|\theta_{\max}|$, we thus have $\widehat{\theta}_{\min}\geq(1-\epsilon)\theta_{\max}$. Now by this modification, we can continue the proof with the procedure described in~\cite{li2016Nonconvex}. Let $I:=\frac{\|\theta'-\theta^*\|_2^2 - \|\theta- \theta^*\|_2^2}{\|\theta- \theta^*\|_2^2}$. For each case, we have::
\begin{itemize}
\item case 1: $I\leq \frac{\theta_{\max}^2}{(r-r^*+r^{**} )(1-\epsilon)\theta^2_{\min}-\theta^2_{\max}}\leq \frac{1}{(r-r^*+r^{**} )(1-\epsilon)-1}$.
\item case 2: $I\leq \frac{r^{**}\theta^2_{\max}}{(r-r^*+r^{**})(1-\epsilon)\theta^2_{\min}}\leq\frac{r^{**}}{(r-r^*+r^{**})(1-\epsilon)}$.
\item case 3:  $I\leq\frac{2\gamma\sqrt{r^{**}}\theta^2_{\max}-r^{**}\theta^2_{\max}}{(r-r^*+r^{**})(1-\epsilon)\theta^2_{\min}+r^{**}\theta^2_{\max}+\gamma^2\theta^2_{\max}-2\gamma\sqrt{r^{**}}\theta^2_{\max}}\\
~~~~~~~~~~~~~\overset{e_1}\leq\frac{2\sqrt{r^{**}}}{2\sqrt{(r-r^*)(1-\epsilon) +r^{**}(\frac{5}{4}-\epsilon)}-\sqrt{r^{**}}}$ \ \ , for some $\gamma\geq\sqrt{r^{**}}$.
\end{itemize}
In all the above cases, we have used the fact that $\widehat{\theta}_{\min}\geq(1-\epsilon)\theta_{\max}$. In addition, $e_1$ in case 3 holds due to maximizing the R.H.S. with respect to $\gamma$. After taking derivative w.r.t. $\gamma$, setting to zero, and solving the resulted quadratic equation, we obtain that:
$$ \gamma = \max\Big{\{} \sqrt{r^{**}},\frac{\sqrt{r^{**}}}{2} +\sqrt{(r-r^*)(1-\epsilon) +r^{**}(\frac{5}{4}-\epsilon)} \Big{\}} $$
Now if we plug in the value of $\gamma$ in the R.H.S of case 3, we obtain the claimed bound. Putting the three above bounds all together, we have: 
\begin{align*}
\frac{\|\theta'-\theta^*\|_2^2 - \|\theta- \theta^*\|_2^2}{\|\theta- \theta^*\|_2^2}&\leq\max\Big{\{}\frac{1}{(r-r^*+r^{**} )(1-\epsilon)-1},\frac{r^{**}}{(r-r^*+r^{**})(1-\epsilon)}, \\
&\hspace{3cm}\frac{2\sqrt{r^{**}}}{2\sqrt{(r-r^*)(1-\epsilon) +r^{**}(\frac{5}{4}-\epsilon)}-\sqrt{r^{**}}}\Big{\}}\\
&\overset{e_1}\leq\frac{2\sqrt{r^{**}}}{2\sqrt{(r-r^*)(1-\epsilon) +r^{**}(\frac{5}{4}-\epsilon)}-\sqrt{r^{**}}}\\
&\overset{e_2}\leq\frac{2\sqrt{r^{*}}}{2\sqrt{(r-r^*)(1-\epsilon) }-\sqrt{r^{*}}}\\
&\overset{e_3}\leq\frac{2}{\sqrt{1-\epsilon}}\frac{\sqrt{r^*}}{\sqrt{r-r^*}},
\end{align*}
where $e_1$ follows by choosing $r$ sufficiently large and the fact that $\epsilon$ can be chosen arbitrary small (this inceases the running time of the approximate projection by $\log(\frac{1}{\epsilon})$ factor), $e_2$ holds due to $r^{**}\leq r^*$, and finally $e_3$ holds by the assumption on $r$ in the lemma. This completes the proof.
\end{proof} 
 
\begin{proof}[Proof of Theorem~\ref{AppSVD}]
Let $V^t, V^{t+1}$, and $V^{*}$ denote the bases for the column space of $L^t, L^{t+1}$, and $L^*$, respectively. Assume $\nu = \sqrt{1+\frac{2}{\sqrt{1-\epsilon}}\frac{\sqrt{r^*}}{\sqrt{r-r^*}}}$. Also, by the definition of the tail projection, we have $L^{t}\in\M(\U_{r})$, and by definition of set $J$ in the theorem, $V^t\cup V^{t+1}\cup V^*\subseteq J_t := J$ such that $rank(J_t)\leq 2r+r^*\leq3r$. Define $b = L^t -\eta\P_J\nabla F(L^t)$. We have:
\begin{align}
\label{linApp}
\|L^{t+1} - L^*\|_F&\overset{e_1}{\leq}\nu\|b-L^*\|_F \nonumber \\
& \leq\nu\|L^t - L^*-\eta\P_J\nabla F(L^t)\|_F \nonumber \\
&\overset{e_2}{\leq}\nu\|L^t - L^*-\eta\P_J\left(\nabla F(L^t)- \nabla F(L^*) \right)\|_F + \nu\eta\|\P_J\nabla F(L^*)\|_F \nonumber \\
&\overset{e_3}{\leq}\nu\sqrt{1+M_{2r+r^*}^2\eta^2 - 2m_{2r+r^*}\eta}\|L^t-L^*\|_F + \nu\eta\|\P_J\nabla F(L^*)\|_F
%&\leq\left(1+2\frac{\sqrt{r^*}}{\sqrt{r-r^*}}\right)\sqrt{1+M_{2r+r^*}^2\eta^2 - 2m_{2r+r^*}\eta}\|L^t-L^*\|_F + \nu\eta\|\P_J\nabla F(L^*)\|_F,
\end{align}
where $e_1$ holds due to applying lemma~\ref{HT}. Moreover, $e_2$ holds by applying triangle inequality and $e_3$ is obtained by combining the lower bound in~\eqref{rscrss_app2} and upper bound in~\eqref{rscrss_app4}, i.e.,
$$\|L^t - L^* - \eta^{\prime}\left(\nabla_J F(L^t) -\nabla_J F(L^*)\right)\|_2^2\leq(1+{\eta^{\prime}}^2M_{2r+r^*}^2-2\eta^{\prime} m_{2r+r^*})\|L^t-L^*\|_2^2.$$
In order that~\eqref{linApp} implies convergence, we require that 
$$\rho = \left(\sqrt{1+\frac{2}{\sqrt{1-\epsilon}}\frac{\sqrt{r^*}}{\sqrt{r-r^*}}}\right)\sqrt{1+M_{2r+r^*}^2\eta^2 - 2m_{2r+r^*}\eta}<1$$. 
By solving this quadratic inequality with respect to $\eta$, we obtain:
\begin{align*}
\left(\frac{M_{2r+r^*}}{m_{2r+r^*}}\right)^2\leq1+\frac{\sqrt{r-r^*}\sqrt{1-\epsilon}}{2\sqrt{r^*}},
\end{align*} 

As a result, we obtain the the condition $r\geq \frac{C_1}{1-\epsilon}\left(\frac{M_{2r+r^*}}{m_{2r+r^*}}\right)^4r^*$ for some $C_1>0$. Furthermore, since $r = \alpha r^*$ for some $\alpha> 1$, we conclude the condition on step size $\eta$ as $\frac{1-\sqrt{\alpha'}}{M_{2r+r^*}}\leq\eta\leq\frac{1+\sqrt{\alpha'}}{m_{2r+r^*}}$ where $\alpha' = \frac{\sqrt{\alpha-1}}{\sqrt{1-\epsilon}\sqrt{\alpha-1}+2}$. This completes the proof of Theorem~\ref{AppSVD}.

\end{proof}

\subsection{Proof of theorems in section~\ref{NLARMsec} }
We first prove the statistical error rate, staing in Theorem~\ref{staterrorNLARM}. 
\begin{proof}[proof of Theorem~\ref{staterrorNLARM}]
Let $b_i = vec(A_i)\in\R^{p^2}$ denotes the $i^{th}$ row of matrix $X\in\R^{m\times p^2}$, defining in the section~\ref{NLARMsec} for $i = 1,\ldots,n$. Since $X$ is constructed by uniform randomly chosen $m$ rows of a $p^2\times p^2$ DFT matrix multiplied by a diagonal matrix whose diagonal entries are uniformly distributed over $\{-1,+1\}^{P^2}$, $\frac{1}{\sqrt{n}}X$ satisfies the rank-$r$ RIP condition with probability at least $1-\exp(-cn\varpi^2)$ ($c>0$ is a constant) provided that $m= \O(\frac{1}{\varpi^2}pr\text{\polylog(p)})$~\cite{candes2011tight}. On the other hand, if a matrix $B$ satisfies the rank-r RIP condition, then~\cite{lee2010admira}
\begin{align}\label{ripprop}
\Big{\|}\P_U\frac{1}{\sqrt{n}}B^*a\Big{\|}_2\leq (1+\delta_r)\|a\|_2, \ \ for \ all \ a\in\R^n,
\end{align}
where $U$ denotes the set of rank-$r$ matrices, and $\delta_{r}$ is the RIP constant. As a result, for all $t=1,\ldots,T$ we have:
$$ \Big{\|}\frac{1}{n}\P_{J_t}\nabla F(L^{*})\Big{\|}_F = \Big{\|}\frac{1}{n}\A^*e\Big{\|}_F = \frac{1}{\sqrt{n}}\Big{\|}\P_{J_t}\frac{1}{\sqrt{n}}X^*e\Big{\|}_2\leq \frac{1+\delta_{2r+r^*}}{\sqrt{n}}\|e\|_2,$$
where the last inequality holds due to~\eqref{ripprop} ($\frac{1}{\sqrt{n}}X$ has RIP constant $\delta_r$, and from our definition, $rank(J_t)\leq 2r+r^*$), and the fact that $e\in\R^n$.
\end{proof}

\begin{proof}[proof of corollary~\ref{induclin}]
Consider upper bound in~\eqref{linApp}. By using induction, zero initialization, and Theorem~\ref{staterrorNLARM}, we obtain $\vartheta$ accuracy after $T_{iter} = \mathcal{O}\left(\log\left(\frac{\|L^*\|_F}{\vartheta}\right)\right)$ iterations. In other words, after $T_{iter}$ iterations, we obtain:
$$ \|L^{T+1} - L^{*}\|_F\leq \vartheta + \frac{1}{\sqrt{n}}\frac{\nu\eta(1+\delta_{2r+r^*})}{1-\rho}\|e\|_2.$$
\end{proof}

The above results shows the linear convergence of APRM if there is no additive noise. We now prove that the objective function defined in problem~\eqref{opt_probNLARM} satisfies the RSC/RSS conditions in each iteration.
\begin{proof}[proof of Theorem~\ref{RSCRSSappNLARM}]
Let $L = L^t$ for all $t=1,\ldots,T$. We follow the approach in~\cite{soltani2016fastIEEETSP17}. hence, we use the the Hessian based definition of RSC/RSC, stating in equation~\eqref{rscrss_app3} in definition~\ref{defRSCRSS_app}. We note that the Hessian of $F(L)$ is given by:
\[
\nabla^2F(L) = \frac{1}{n}\sum_{i=1}^{n}A_ig'(\langle A_i,L\rangle)A_i^T,
\]
According to our assumption on the link function, we know $0<\mu_1\leq g'(x)\leq\mu_2$ for all $x\in\mathcal{D}(g)$. As a result $\lambda_{\min}(\nabla^2F(L))\geq0$ due to the positive semidefinite of $A_iA_i^T$ for all $i=1,\dots,n$. Now let $\Lambda_{\max} =\max_U \lambda_{\max}(\P_U\nabla^2F(L))$ and $\Lambda_{\min} =\min_U \lambda_{\min}(\P_U\nabla^2F(L))$. Moreover, let $W$ be any set of rank-$2r$ matrices such that $U\subseteq W$. We have:
\begin{align}\label{RIPNLARM}
\mu_1\min_W \lambda_{\min}\left(\P_W\left(\frac{1}{n}\sum_{i=1}^{n}A_iA_i^T\right)\right)\leq\Lambda_{\min}
\leq\Lambda_{\max}\leq\mu_2\max_W \lambda_{\max}\left(\P_W\left(\frac{1}{n}\sum_{i=1}^{n}A_iA_i^T\right)\right),
\end{align}
Now, we need to bound the upper bound and the lower bound in the above inequality. To do this, we are using the assumption on the design matrices $A_i$'s, stating in the theorem. According to this, we can write, $\P_W\left(\frac{1}{n}\sum_{i=1}^{n}A_iA_i^T\right) = \P_W\left(\frac{1}{n}X^TX\right)$. We follow the approach of~\cite{HegdeFastUnionNips2016}. Now fix any set $W$ as defined above. Recall that $X =X'D$, where $X'$ is a partial Fourier or partial Hadamard matrix. Thus, by~\cite{haviv2017restricted}, $X'$ satisfies RIP condition with constant $4\upsilon$ over the set of of $s$-sparse vectors with high probability when $m=\O\left(\frac{1}{\upsilon^2}s\log^2(\frac{s}{\upsilon})\log(p)\right)$. Also from~\cite{krahmer2011new}, $X$ is a $(1\pm\xi)-$Johnson-Lindenstrauss embedding (with $4\upsilon<\xi$) for set $W$ with probability at least $1-\varsigma$ provided that $s >\O(\frac{V}{\varsigma})$, where $V$ is the number of vectors in $W$. In other words, the Euclidean distance between any two vectors (matrix) $\beta_1,\beta_2\in W\in\R^{p\times p}$ is preserved up to a $\pm\xi$ by application of $X$. As a result, with high probability
\[
1-\xi\leq\lambda_{\min}\left(\P_W\left(\frac{1}{n}\sum_{i=1}^{n}A_iA_i^T\right)\right)\leq\lambda_{\max}\left(\P_W\left(\frac{1}{n}\sum_{i=1}^{n}A_iA_i^T\right)\right)\leq 1+\xi
\]
Now it remains to argue the final bound in~\eqref{RIPNLARM}. By~\cite{candes2011tight}, we know that the set of $p\times p$ rank-$r$ matrices can be discretized by a $\zeta$-cover $S_r$ such that $|S_r| = (\frac{9}{\zeta})^{(2p+1)r}$. In addition, They show that if a matrix $X$ satisfies JL embedding by constant $\xi$, then $X$ satisfies the rank-$r$ RIP with constant $\omega=\O(\xi)$. As a result, by taking union bound (taking maximum over all set $W$ in~\eqref{RIPNLARM}), we establish RSC/RSS constants such that $M_{2r+r^*}\leq \mu_2(1+\omega)$ and $m_{2r+r^*}\geq \mu_1(1-\omega)$ provided that $s=\O(pr)$ and $V = |S_r|$ which implies $m= \O(pr\text{\polylog(p)})$. Now, In order to satisfy the assumptions in Theorem~\ref{AppSVD}, we need to have $\frac{M_{2r+r^*}^2}{m_{2r+r^*}^4}\leq C_2(1-\epsilon)\frac{r}{r^*}$ for some $C_2>0$ and $\epsilon$ defined in lemma~\ref{ApprHTh}. Thus, we have $\frac{\mu_2^4(1+\omega)^4}{\mu_1^4(1-\omega)^4}\leq C_2(1-\epsilon)\frac{r}{r^*}$ which justifies the assumption in Theorem~\ref{RSCRSSappNLARM}. 
\end{proof} 

\subsection{Proof of theorems in section~\ref{PME} }
\begin{proof}[Proof of Theorem~\ref{ExactSVD}]
%In all the following equations, $C_1,C_2$ and $C_3$ are some positive constants. 
Let $V^t, V^{t+1}$, and $V^{*}$ denote the bases for the column space of $L^t, L^{t+1}$, and $L^*$, respectively. Assume $\nu' = \sqrt{1+\frac{2\sqrt{r^*}}{\sqrt{r-r^*}}}$. By definition of set $J$ in the theorem, $V^t\cup V^{t+1}\cup V^*\subseteq J_t := J$ and $\rank(J_t)\leq 2r+r^*$. Define $b = L^t -\eta'\P_J\nabla F(L^t)$. We have:
\begin{align}
\label{linExact}
\|L^{t+1} - L^*\|_F&\overset{e_1}{\leq}\nu'\|b-L^*\|_F \nonumber \\
& \leq\nu\|L^t - L^*-\eta'\P_J\nabla F(L^t)\|_F \nonumber \\
&\overset{e_2}{\leq}\nu'\|L^t - L^*-\eta'\P_J\left(\nabla F(L^t)- \nabla F(L^*) \right)\|_F + \nu'\eta'\|\P_J\nabla F(L^*)\|_F \nonumber \\
&\overset{e_3}{\leq}\nu'\sqrt{1+M_{2r+r^*}^2\eta'^2 - 2m_{2r+r^*}\eta'}\|L^t-L^*\|_F + \nu'\eta'\|\P_J\nabla F(L^*)\|_F,
%&\leq\left(1+2\frac{\sqrt{r^*}}{\sqrt{r-r^*}}\right)\sqrt{1+M_{2r+r^*}^2\eta^2 - 2m_{2r+r^*}\eta}\|L^t-L^*\|_F + \nu\eta\|\P_J\nabla F(L^*)\|_F,
\end{align}
where $e_1$ holds due to applying lemma~\ref{HT}. Moreover, $e_2$ holds by applying triangle inequality and $e_3$ is obtained by combining the lower bound in~\eqref{rscrss_app2} and upper bound in~\eqref{rscrss_app4}, i.e.,
$$\|L^t - L^* - \eta^{\prime}\left(\nabla_J F(L^t) -\nabla_J F(L^*)\right)\|_2^2\leq(1+{\eta^{\prime}}^2M_{2r+r^*}^2-2\eta^{\prime} m_{2r+r^*})\|L^t-L^*\|_2^2.$$
%together with the inequality $\|AB\|_F\leq\|A\|_2\|B\|_F$ for any matrix $A$ and $B$. 
In order that~\eqref{linExact} implies convergence, we require that 
$$\rho' = \sqrt{1+2\frac{\sqrt{r^*}}{\sqrt{r-r^*}}}\sqrt{1+M_{2r+r^*}^2\eta'^2 - 2m_{2r+r^*}\eta'}<1$$. 
By solving this quadratic inequality with respect to $\eta$, we obtain:
\begin{align*}
\left(\frac{M_{2r+r^*}}{m_{2r+r^*}}\right)^2\leq1+\frac{\sqrt{r-r^*}}{2\sqrt{r^*}},
\end{align*} 

As a result, we obtain the the condition $r\geq C_1'\left(\frac{M_{2r+r^*}}{m_{2r+r^*}}\right)^4r^*$ for some $C_1'>0$. Furthermore, since $r = \alpha r^*$ for some $\beta> 1$, we conclude the condition on step size $\eta'$ as $\frac{1-\sqrt{\beta'}}{M_{2r+r^*}}\leq\eta'\leq\frac{1+\sqrt{\beta'}}{m_{2r+r^*}}$ where $\beta' = \frac{\sqrt{\beta-1}}{\sqrt{\beta-1}+2}$ for some $\beta>1$.
If we initialize at $L^0 = 0$, then we obtain $\vartheta$ accuracy after $T = \mathcal{O}\left(\log\left(\frac{\|L^*\|_F}{\vartheta}\right)\right)$ iterations.
\end{proof}

\begin{proof}[Proof of Theorem~\ref{BoundGrad}]
The proof of this theorem is a direct application of the Lemma 5.4 in \cite{Venkat2009sparse} and we restate it for completeness:
\begin{lemma}%[eviation of the sample covariance from the true covariance matrix]
\label{boundSmCo}
Let $C$ denote the sample covariance matrix, then with probability at least $1 - 2\exp(-p)$ we have $\|C - (S^* + L^*)^{-1}\|_2\leq c_1\sqrt{\frac{p}{n}}$ where $c_1>0$ is a constant.
\end{lemma}
By noting that $\nabla F(L^{*}) = C - (S^* + L^*)^{-1}$ and $\rank(J_t)\leq 2r+r^*\leq3r$, we can bound the term on the right hand side in Theorem~\ref{ExactSVD} as:
$$\|\P_{J_t}\nabla F(L^{*})\|_F\leq \sqrt{3r}\|\nabla F(L^{*})\|_2\leq c_2\sqrt{\frac{rp}{n}} .$$
%The proof for upper-bounding $\|\P_{V_t}\nabla F(L^{*})\|_F$ in Theorem~\ref{AppSVD} follows analogously. 
%Similarly, since outputted subspace by head projection, $V$ contains matrices with rank $2r$, we can bound the Frobenius norm in theorem~\ref{AppSVD} as
%$$\|\P_V\nabla F(L^{*})\|_F\leq \sqrt{2r}\|\nabla F(L^{*})\|_2\leq c_3\sqrt{\frac{rp}{n}}.$$
\end{proof}

The key observation is that the objective function in~\eqref{opt_probPME} is globally strongly convex, and when restricted to any compact psd cone, it also satisfies the smoothness condition. As a result, it satisfies RSC/RSS conditions. Our strategy to prove Theorems~\ref{RSCRSSex} and~\ref{RSCRSSapp} is to establish upper and lower bounds on the spectrum of the sequence of estimates $L^t$ independent of $t$. We use the following lemma.
\begin{lemma}\cite{yuan2013gradient,boyd2004convex}
\label{hessian}
The Hessian of the objective function $F(L)$ is given by $\nabla^2F(L) = \Theta^{-1}\otimes\Theta^{-1}$ where $\otimes$ denotes the Kronecker product and $\Theta = \bS + L$. In addition if $\alpha I\preceq\Theta\preceq\beta I$ for some $\alpha$ and $\beta$, then $\frac{1}{\beta^2} I\preceq\nabla^2F(L)\preceq\frac{1}{\alpha^2} I$.
\end{lemma}
\begin{lemma}[Weyl type inequality]
\label{weyl}
For any two matrices $A,B\in\R^{p\times p}$, we have:
$$\max_{1\leq i\leq p} |\sigma_i(A+B) - \sigma_i(A)|\leq\|B\|_2.$$
%$$\lambda_i(A) - \lambda_1(B)\leq\lambda_i(A-B)\leq\lambda_i(A) - \lambda_p(B)\quad\ i=1\dots p.$$
\end{lemma}

If we establish an universal upper bound and lower bound on $\lambda_1(\Theta^t)$ and $\lambda_p(\Theta^t)$ for all $t=1\dots T$, then we can bound the RSC constant as $m_{2r+r^*}\geq\frac{1}{\lambda_1(\Theta^t)^2}$ and the RSS-constant as $M_{2r+r^*}\leq\frac{1}{\lambda_p(\Theta^t)^2}$ using Lemma~\ref{hessian} and the definition of RSS/RSC.

\begin{proof}[Proof of Theorem~\ref{RSCRSSex}]
Recall that by Theorem~\ref{ExactSVD}, we have: 
$$\|L^{t} - L^{*}\|_F\leq \rho'\|L^{t-1} - L^{*}\|_F + \nu'\eta'\|\P_{J_t}\nabla F(L^{*})\|_F,$$
By Theorem~\ref{BoundGrad}, the second term on the right hand side can be bounded by $\O(\sqrt{\frac{rp}{n}})$ with high probability. Therefore, recursively applying this inequality to $L^t$ (and initializing with zero), we obtain:
\begin{align}
\label{firstexactupp}
\|L^{t} - L^{*}\|_F\leq(\rho')^t\|L^*\|_F+ \frac{c_2\nu'\eta'}{1-\rho'} \sqrt{\frac{rp}{n}}.
\end{align} 
Since $\rho'<1$, then $ (\rho')^t<1$. On the other hand $\|L^*\|_F\leq\sqrt{r^*}\|L^*\|_2$. Hence, $\rho^t\|L^*\|_F\leq\sqrt{r^*}\|L^*\|_2$. Also, by the Weyl inequality, we have:
\begin{align}
\|L^t\|_2 - \|L^*\|_2\leq\|L^t-L^*\|_2\leq\|L^t-L^*\|_F.
\end{align}
Combining~\eqref{firstexactupp} and~\eqref{secondexactupp} and using the fact that $\la_1(L^{t})\leq\si_1(L^{t})$,  
\begin{align}
\la_1(L^{t})&\leq\|L^*\|_2+\|L^t-L^*\|_F\nonumber  \\
&\leq\|L^*\|_2+\sqrt{r^*}\|L^*\|_2+ \frac{c_2\nu'\eta'}{1-\rho'} \sqrt{\frac{rp}{n}}.\nonumber  
\end{align}
Hence for all $t$, 
\begin{align}
\label{secondexactupp}
\la_1(\Theta^t) = S_1+\la_1(L^t)\leq S_1 +\left(1+\sqrt{r^*}\right)\|L^*\|_2+\frac{c_2\nu'\eta'}{1-\rho'} \sqrt{\frac{rp}{n}}.
\end{align}
For the lower bound, we trivially have for all $t$:
\begin{align}
\la_p(\Theta^t) = \la_p(\bS +  L^t)\geq S_p.
\end{align}
If we select $n=\O\left(\frac{1}{\delta^2}\left(\frac{\nu'\eta'}{1-\rho'}\right)^2rp\right)$ for some small constant $\delta>0$, then~\eqref{secondexactupp} becomes:
\begin{align*}
\la_1(\Theta^t)\leq S_1 +\left(1+\sqrt{r}\right)\|L^*\|_2+ \delta.
\end{align*}
As mentioned above, we set $m_{2r+r^*}\geq\frac{1}{\lambda_1^2(\Theta^t)}$ and $M_{2r+r^*}\leq\frac{1}{\lambda_p^2(\Theta^t)}$ which implies $\frac{M_{2r+r^*}}{m_{2r+r^*}}\leq\frac{\la_1^2(\Theta^t)}{\la_p^2(\Theta^t)}$. In order to satisfy the assumption on the RSC/RSS in theorem~\ref{ExactSVD}, i.e., $\frac{M_{2r+r^*}^4}{m_{2r+r^*}^4}\leq C_2'\frac{r}{r^*}$ for some $C_2'>0$, we need to establish a regime such that $\frac{\la_1^8(\Theta^t)}{\la_p^8(\Theta^t)}\leq C_2'\frac{r}{r^*}$. 
As a result, to satisfy this condition, we need to have the following condition, verifying the assumption in the theorem. % on $S_1$ and $S_p$: 
\begin{align}
S_p\leq S_1\leq C_3'(\frac{r}{r^*})^{\frac{1}{8}}S_p - \left(1+\sqrt{r^*}\right)\|L^*\|_2- \delta.
\end{align}
for some constant $C_3'>0$.  
%This condition says that the diagonal part, $S^*$ should be well-posed matrix.
\end{proof}

\begin{proof}[Proof of Theorem~\ref{RSCRSSapp}]
The proof is similar to the proof of theorem~\ref{RSCRSSex}.  Recall that by theorem~\ref{AppSVD}, we have 
$$\|L^{t+1} - L^{*}\|_F\leq \rho\|L^{t} - L^{*}\|_F + \nu\eta\|\P_{J_t}\nabla F(L^*)\|_F,$$ 
As before, the second term on the right hand side is bounded by $\O(\sqrt{\frac{rp}{n}})$ with high probability by Theorem~\ref{BoundGrad}. As above, recursively applying this inequality to $L^t$ and using zero initialization, we obtain:
\begin{align*}
\|L^{t} - L^{*}\|_F\leq\rho^t\|L^*\|_F  + \frac{c_2\nu\eta}{1-\rho}\sqrt{\frac{rp}{n}}.
\end{align*}
Since $\rho<1$, then $\rho^t<1$. Now similar to the exact algorithm, $\|L^*\|_F\leq\sqrt{r^*}\|L^*\|
|_2$ and $\rho_1^t\|L^*\|_F\leq\sqrt{r^*}\|L^*\|_2$. , Hence with high probability,
\begin{align}
\label{secontineqApp}
\la_1(L^{t})&\leq\|L^*\|_2+\|L^t-L^*\|_F\nonumber  \\
%&\leq\|L^*\|_2 + \sqrt{r}\|L^*\|_2 + \frac{c_2\nu\eta}{1-\rho}\sqrt{\frac{rp}{m}} \nonumber  \\
&\overset{}{\leq}\left(1+\sqrt{r^*}\right)\|L^*\|_2+ \frac{c_2\nu\eta}{1-\rho}\sqrt{\frac{rp}{n}},
\end{align} 
Hence, for all $t$:
\begin{align}
\label{firstineqApp}
\la_1(\Theta^t) = S_1+\la_1(L^t)\leq S_1 + \left(1+\sqrt{r^*}\right)\|L^*\|_2+ \frac{c_2\nu\eta}{1-\rho}\sqrt{\frac{rp}{n}},
\end{align}
%\new{%According to the assumption~\ref{mineigLapp},} 
Also, we trivially have:
\begin{align}
\la_p(\Theta^t)=\la_p(\bS +  L^t)\geq S_p - a^{\prime}, \ \forall t.
\end{align}
By selecting $n=\O\left(\frac{1}{\delta^{\'2}}\left(\frac{\nu\eta}{1-\rho}\right)^2rp\right)$ for some small constant $\delta^{\'}>0$, \eqref{firstineqApp} can be written as follows:
\begin{align*}
\la_1(\Theta^t)\leq S_1 + \left(1+\sqrt{r^*}\right)\|L^*\|_2+\delta^{\'},
\end{align*}
In order to satisfy the assumptions in Theorem~\ref{AppSVD}, i.e., $\frac{M_{2r+r^*}^4}{m_{2r+r^*}^4}\leq \frac{C_2''}{1-\epsilon}\frac{r}{r^*}$, we need to guarantee that $\frac{\la_1^8(\Theta^t)}{\la_p^8(\Theta^t)}\leq \frac{C_2''}{1-\epsilon}\frac{r}{r^*}$.  As a result, to satisfy this inequality, we need to have the following condition on $S_1$ and $S_p$: 
\begin{align}
S_p\leq S_1&\leq \frac{C_3''}{(1-\epsilon)^{\frac{1}{8}}}(\frac{r}{r^*})^{\frac{1}{8}}(S_p-a^{\'}) - \left(1+\sqrt{r^*}\right)\|L^*\|_2- \delta^{\'}.
\end{align}
for some $C_3''>0$. Also, we can choose RSC/RSS constant as previous case.
\end{proof}

%Finally, we prove theorem~\ref{mineigval} which says that the minimum eigenvalue of outputted matrix in each iteration by AP-LVGGM does not have large negative eigenvalue with high probability.

\begin{proof}[Proof of Theorem~\ref{mineigval}]
Recall from~\eqref{secontineqApp} that with very high probability, 
$$\|L^t\|_2{\leq}\left(1+\sqrt{r^*}\right)\|L^*\|_2+ \frac{c_2\nu\eta}{1-\rho}\sqrt{\frac{rp}{n}}.$$ 
Also, we always have: $\la_p(L^t)\geq-\|L^t\|_2$. As a result:
\begin{align}
\la_p(L^t)\geq-\left(1+\sqrt{r^*}\right)\|L^*\|_2- \frac{c_2\nu\eta}{1-\rho}\sqrt{\frac{rp}{n}}.
\end{align}
Now if the inequality $\left(1+\sqrt{r^*}\right)\|L^*\|_2+ \frac{c_2\nu\eta}{1-\rho}\sqrt{\frac{rp}{n}}< S_p$ is satisfied, then we can select $0<a^{\'}\leq\left(1+\sqrt{r^*}\right)\|L^*\|_2+ \frac{c_2\nu\eta}{1-\rho}\sqrt{\frac{rp}{n}}$. The former inequality is satisfied 
%by choosing $n$ large enough, i.e., $n=\O\left(\frac{1}{\delta^{\'2}}\left(\frac{\rho_2}{1-\rho_1}\right)^2rp\right)$ and 
by the assumption of Theorem~\ref{RSCRSSapp} on $\|L^*\|_2$, i.e., 
\begin{align*}
\|L^*\|_2\leq\frac{1}{1+\sqrt{r^*}}\left(\frac{S_p}{1+C_4\left((1-\epsilon)(\frac{r^*}{r})\right)^{\frac{1}{8}}} - \frac{S_1(C_4((1-\epsilon)(\frac{r^*}{r}))^{\frac{1}{8}})}{1+C_4\left((1-\epsilon)(\frac{r^*}{r})\right)^{\frac{1}{8}}} - \frac{c_2\nu\eta}{1-\rho}\sqrt{\frac{rp}{n}}\right).
\end{align*}
\end{proof}